\documentclass[runningheads,review]{llncs}
\pdfoutput=1

\usepackage[utf8]{inputenc} % allow utf-8 input
\usepackage[T1]{fontenc}    % use 8-bit T1 fonts
\usepackage[colorlinks=false,urlcolor=blue]{hyperref}       % hyperlinks
\usepackage{url}            % simple URL typesetting
\usepackage{booktabs}       % professional-quality tables                                              
\usepackage{amsfonts}       % blackboard math symbols
\usepackage{amsmath}
\usepackage{amssymb}
\usepackage{xpatch}
\usepackage{graphicx}
\usepackage{txfonts}
\usepackage{wrapfig}
\usepackage{lipsum}
\usepackage{threeparttable}
\usepackage{nicefrac}       % compact symbols for 1/2, etc.
\usepackage{microtype}      % microtypography
\usepackage[ruled,vlined,linesnumbered]{algorithm2e}

\SetCommentSty{mycommfont}
\usepackage{verbatim}
\usepackage{multirow}
\usepackage{diagbox}
\usepackage{siunitx}
\usepackage{bbding}
\usepackage{float}
\usepackage{etoolbox}
\usepackage{subcaption}
\usepackage{makecell}
\usepackage{wrapfig}
\usepackage{booktabs}
\usepackage{bigstrut}
\usepackage{tabu}
\usepackage{tikz}
\newcommand{\occrob}{\textsc{OccRob}\ }
\newcommand{\occrobnoindent}{\textsc{OccRob}}
\usepackage{todonotes}

\makeatletter
\patchcmd{\@algocf@start}% <cmd>
{-1.5em}% <search>
{0pt}% <replace>
{}{}% <success><failure>
\newcounter{proofpart}
\xpretocmd{\proof}{\setcounter{proofpart}{0}}{}{}
\newcommand{\proofpart}[1]{%
  \par
  \addvspace{\medskipamount}%
  \stepcounter{proofpart}%
  \noindent\emph{Case \theproofpart: #1}\par\nobreak\smallskip
  \@afterheading
}
\makeatother
\def\HiLi{\leavevmode\rlap{\hbox to \hsize{\color{gray!30}\leaders\hrule height .6\baselineskip depth .5ex\hfill}}}
\usepackage{array}
\newcolumntype{L}[1]{>{\raggedright\let\newline\\\arraybackslash\hspace{0pt}}m{#1}}
\newcolumntype{C}[1]{>{\centering\let\newline\\\arraybackslash\hspace{0pt}}m{#1}}
\newcolumntype{R}[1]{>{\raggedleft\let\newline\\\arraybackslash\hspace{0pt}}m{#1}}

\SetKwInput{KwInput}{Input}                % Set the Input
\SetKwInput{KwOutput}{Output}              % set the Output

\usepackage{algcompatible}

\definecolor{light-gray}{gray}{0.8}

\graphicspath{{Figures/}}
\begin{document}

\title{\textsc{OccRob}: Efficient SMT-Based Occlusion Robustness Verification of Deep Neural Networks}
\author{Xingwu Guo $^{1}$, Ziwei Zhou $^{1}$, Yueling Zhang $^{1}$, Guy Katz $^{2}$, Min Zhang $^{1}$}
\institute{$^{1}$ Shanghai Key Laboratory of Trustworthy Computing,\\ East China Normal University, Shanghai, China\\
$^{2}$ The Hebrew University of Jerusalem, Jerusalem, Isarel
%\\
%\texttt{zhangmin@sei.ecnu.edu.cn}
}
\maketitle

\begin{abstract}
Occlusion is a prevalent and easily realizable semantic perturbation to deep neural networks (DNNs). It can fool a DNN into misclassifying an input image by occluding some segments, possibly resulting in severe errors. Therefore, DNNs planted in safety-critical systems should be verified to be robust against occlusions prior to deployment. However, most existing robustness verification approaches for DNNs are focused on non-semantic perturbations and are not suited to the occlusion case. 
In this paper, we propose the first efficient, SMT-based approach for formally verifying the occlusion robustness of DNNs. We formulate the occlusion robustness verification problem and prove it is NP-complete. Then, we devise a novel approach for encoding occlusions as a part of neural networks and introduce two acceleration techniques so that the extended neural networks can be efficiently verified using off-the-shelf, SMT-based neural network verification tools. 
%an efficient algorithm based on the state-of-the-art SMT technique. 
We implement our approach in a prototype called \textsc{OccRob} and extensively evaluate its performance on benchmark datasets with various occlusion variants. The experimental results demonstrate our approach's effectiveness and efficiency in verifying DNNs' robustness against various occlusions, and its ability to generate counterexamples when these DNNs are not robust.
%The results show that our approach can perform comprehensive verification on occlusion perturbations and is more efficient than applying naive SMT method directly.
\vspace{-1mm}
\end{abstract}



%\keywords{Neural networks, formal verification, occlusion perturbation}

\section{Introduction}
\vspace{-1mm}
Deep neural networks (DNNs) are computer-trained \textit{programs} that can implement hard-to-formally-specify tasks. They have repeatedly demonstrated their potential in enabling artificial intelligence in various domains, such as face recognition~\cite{cocskun2017face} and autonomous driving~\cite{lillicrap2015continuous}.
They are increasingly being incorporated into safety-critical applications with interactive environments. To ensure the security and reliability of these applications, DNNs must be highly dependable against adversarial and environmental perturbations. 
This dependability property is known as \emph{robustness} and is attracting a considerable amount of research efforts from both academia and industry, aimed at ensuring robustness via different technologies such as adversarial training \cite{gowal2019scalable,lyu2021towards}, testing \cite{tian2018deeptest,pei2017deepxplore}, and formal verification \cite{raghunathan2018certified,gehr2018ai2,cohen2019certified}.

Occlusion is a prevalent kind of perturbation, which may cause DNNs to misclassify an image by occluding some segment thereof  \cite{su2019one,lengyel2019easily,eykholt2018robust}. For instance, a ``turn left'' traffic sign may be misclassified as ``go straight'' after it is occluded by a tape, probably resulting in traffic accidents. 
A similar situation may occur in face recognition, where many well-trained neural networks fail to recognize faces correctly when they are partially occluded, such as when glasses are worn\cite{song2019occlusion}.
A neural network is called \textit{robust against occlusions} if small occlusions do not alter its classification results. Generally, we wish a DNN to be robust against occlusions that appear negligible to humans.

It is challenging to verify whether a DNN is robust or not on an input image if the image is occluded. 
On the one hand, the verification problem is non-convex due to the non-linear activation functions in DNNs. It is NP-complete even when dealing with common, fully connected feed-forward neural networks (FNNs) \cite{katz2017reluplex}. 
On the other hand, unlike existing perturbations, occlusions are challenging to encode using $L_p$ norms. Most existing robustness verification approaches assume that perturbations need to be defined by $L_p$ norms and then apply approximations and abstract interpretation  techniques \cite{raghunathan2018certified,gehr2018ai2,cohen2019certified} as part of the verification process. The semantic effect of occlusions partially alters the values of some neighboring pixels from large to small or in the inverse direction, e.g., 255 to 0, when a black occlusion occludes a white pixel. Therefore, existing techniques for perturbations in $L_p$ norms are not suited to occlusion perturbations.

\vspace{-0.1mm}
SMT-based approaches have been shown to be an efficient approach to DNN verification \cite{katz2017reluplex}. They are both sound and complete, in that they always return definite results and produce counterexamples in non-robust cases. We show that, although it is straightforward to encode the occlusion robustness verification problem into SMT formulas, solving the constraints generated by this na{\"i}ve encoding is experimentally beyond the reach of state-of-the-art SMT solvers, due to the inclusion of a large number of the piece-wise ReLU activation functions. Consequently, such a straightforward encoding approach cannot scale to large networks.

\vspace{-0.1mm}
In this paper, we systematically study the occlusion robustness verification problem of DNNs. We first formalize and prove that the problem is NP-complete for ReLU-based FNNs(see Appendix \ref{appendix:proof-of-np-completeness}). Then, we propose a novel approach for encoding various occlusions and neural networks together to generate new equivalent networks that can be efficiently verified using off-the-shelf SMT-based robustness verification tools such as Marabou \cite{katz2019marabou}. In our encoding approach, although additional neurons and layers are introduced for encoding occlusions, the number is reasonably small and independent of the networks to be verified. 
The efficiency improvement of our approach comes from the fact that our approach significantly reduces the number of constraints introduced while encoding the occlusion and leverages the backend verification tool’s optimization against the neural network structure. Furthermore, we introduce two acceleration techniques, namely input-space splitting to reduce the search space of a single verification, which can significantly improve verification efficiency, and label sorting to help verification terminates earlier.
We implement a tool called \occrob with Marabou as the backend verification tool. 
To our knowledge, this is the first work on formally verifying the occlusion robustness of deep neural networks. 

To demonstrate the effectiveness and efficiency of \occrobnoindent, we evaluate it on six representative FNNs trained on two benchmark datasets. 
The empirical results show that our approach is effective and efficient in verifying various types of occlusions with respect to the occlusion position, size, and occluding pixel value. 

\vspace{-0.1mm}
\noindent\textbf{Contributions.} We make the following three major contributions: (i) we propose a novel approach for encoding occlusion perturbations, by which we can leverage \emph{off-the-shelf} SMT-based robustness verification tools to verify the robustness of neural networks against various occlusion perturbations; 
(ii) we prove the verification problem of the occlusion robustness is NP-complete and introduce two acceleration techniques, i.e., label sorting and input space splitting, to improve the efficiency of verification further; and (iii) we implement a tool called \occrob and conduct experiments extensively on a collection of benchmarks to demonstrate its effectiveness and efficiency. 

\vspace{-0.1mm}
\noindent\textbf{Paper Organization.} Sec. \ref{sec:pre} introduces preliminaries. Sec. \ref{sec:occrob} formulates the occlusion robustness verification problem and studies its complexity. Sec. \ref{sec:encoding} presents our encoding approach and acceleration techniques for the verification. Sec. \ref{sec:exp} shows the experimental results. Sec. \ref{sec:rel} discusses related work, and Sec. \ref{sec:con} concludes the paper.

%%% Local Variables:
%%% mode: latex
%%% TeX-master: "main"
%%% End:

\section{Preliminaries}
\label{sec:pre}

\begin{wrapfigure}{r}{0.48\textwidth}
	\vspace{-10mm}
% \centering
	\hspace{-7mm}\includegraphics[width=0.57\textwidth]{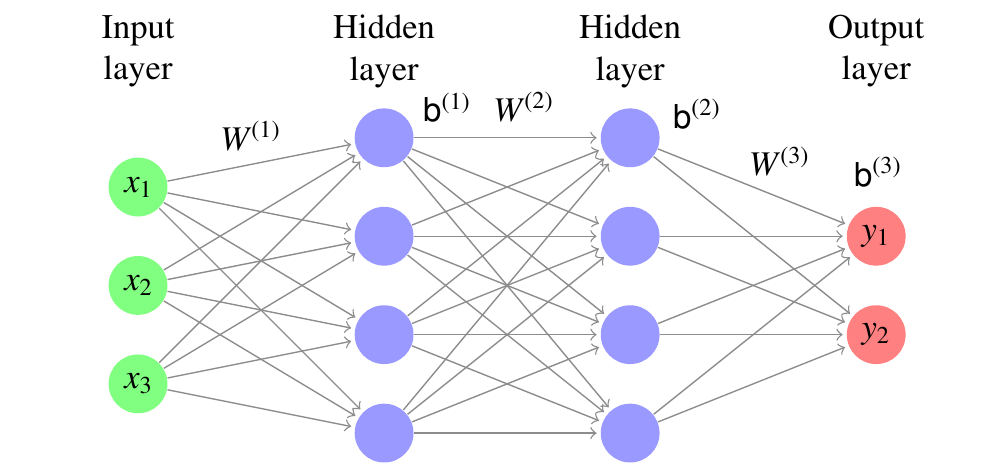}
	\vspace{-4mm}
	\caption{A fully-connected feed-forward neural network (FNN).}
	\label{fig:dnn}
	\vspace{-6mm}
\end{wrapfigure}

\subsection{Deep Neural Networks and the Robustness}
As shown in Fig. \ref{fig:dnn}, a deep neural network
consists of multiple layers. The neurons on the input layer take input values, which are computed and propagated through the hidden layers and then output by the output layer. The neurons on each layer are connected to those on the predecessor and successor layers. We only consider fully connected, feedforward networks (FNNs) \cite{Goodfellow2016}.

Given a $\lambda$-layer neural network, let $W^{(i)}$ be the weight matrix between the $(i-1)$-th and $i$-th layers, and ${\sf b}^{(i)}$ be the biases of the corresponding neurons, where $i=1,2,\ldots,\lambda$. 
The network implements a function $F:\mathbb{R}^u \rightarrow \mathbb{R}^{r}$ that is recursively defined by:  
\begin{align*}  
&\begin{aligned}
&z^{(0)} = x\\
&z^{(i)} = \sigma(W^{(i)} \cdot z^{(i-1)} + {\sf b}^{(i)}),~for\ i=1,\dots, \lambda -1
\end{aligned} \tag{Layer Function}\\
&F(x)=W^{(\lambda)} \cdot z^{(\lambda -1)} + {\sf b}^{(\lambda)}\tag{Network Function} 
\end{align*}
where $\sigma(\cdot)$ is called an \textit{activation function} and $z^{(i)}$ denotes the result of neurons at the $i$-th layer.

For example, Fig. \ref{fig:dnn} shows a 3-layer neural network with three input neurons and two output neurons, namely, $\lambda =3$, $u = 3$ and $r = 2$. 

For the sake of simplicity, we use $\Phi_F(x)= \mathop{arg\  max}_{\ell\in L} F(x)$ to denote the label  $\ell$ such that the probability $F_{\ell}(x)$ of classifying $x$ to $\ell$ is larger than those to other labels, where $L$ represents the set of labels. 
The activation function $\sigma$ usually can be a piece-wise Rectified Linear Unit (ReLU), $\sigma(x)=max(x,0)$), or S-shape functions like Sigmoid $\sigma(x)=\frac{1}{1+e^{-x}}$, Tanh $\sigma(x) = \frac{e^x - e^{-x}}{e^x + e^{-x}}$, or Arctan $\sigma(x) = tan^{-1}(x)$. In this work, we focus on the  networks that only contain ReLU activation functions, which are widely adopted in real-world applications.

A neural network is called \emph{robust} if small perturbations to its inputs do not alter the classification result \cite{szegedy2013intriguing}. Specifically, given a network $F$, an input $x_0$ and a set $\Omega$ of perturbed inputs of $x_0$, $F$ is called locally robust with respect to $x_0$ and $\Omega$ if $F$ classifies all the perturbed inputs in $\Omega$ to the same label as it does $x_0$.

\vspace{-1mm}
\begin{definition}[Local Robustness \cite{huang2017safety}]\label{robust_def}
	A neural network $F:\mathbb{R}^u \rightarrow \mathbb{R}^{r}$ is called \textit{locally robust} with respect to an input $x_0$ and a set $\Omega$ of perturbed inputs of $x$ if $\forall x \in \Omega, \Phi_F(x) = \Phi_F(x_0)$ holds. 
\end{definition}
Usually, the set $\Omega$ of perturbed inputs is defined by an $\ell_p$-norm ball around $x_0$ with a radius of $\epsilon$, i.e., $\mathbb{B}_p (x_0, \epsilon):=\{x\ |\ \| x-x_0\|_p \le \epsilon\}$ \cite{huang2017safety,boopathy2019cnn}.

\vspace{-2mm}
\subsection{Occlusion Perturbation}
\label{section:occlusion perturbation}
\vspace{-1mm}

In the context of image classification networks, occlusion is a kind of perturbation that blocks the pixels in certain areas before the image is fed into the network.
Existing studies showed that the classification accuracy of neural networks could be significantly decreased when the input objects are artificially occluded \cite{kortylewski2021compositional,zhu2019robustness}.

\begin{figure}[t]
	%\hspace{-5mm}
	\centering
	\includegraphics[width=.9\linewidth]{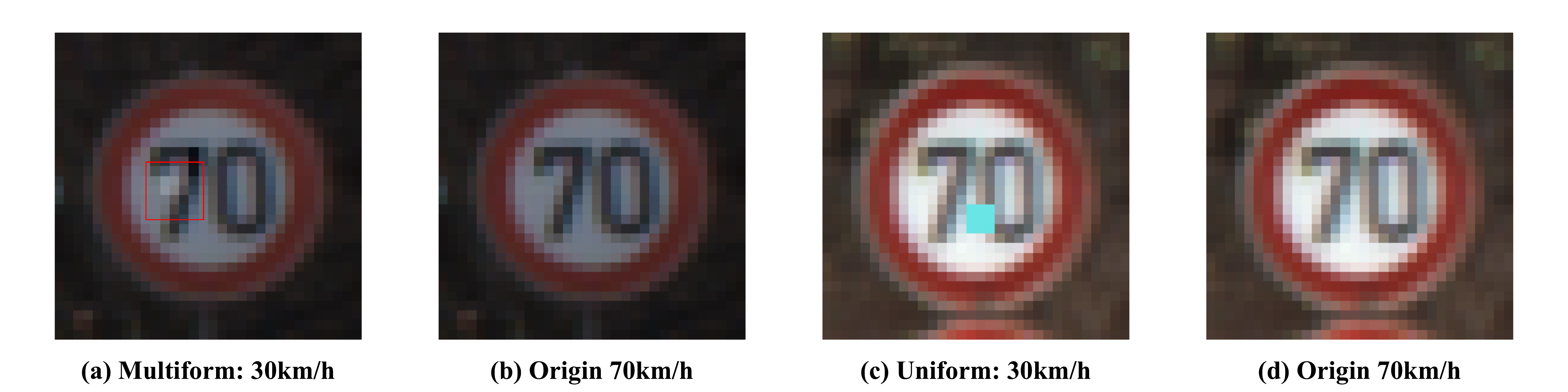}
	\vspace{-1mm}
	\caption{Two multiform and uniform occlusions to traffic signs causing mis-classifications.}
	\label{fig:occlusion example}
	\vspace{-5mm}
\end{figure}

Occlusions can have various occlusion shapes, sizes, colors, and positions. The shapes can be square, rectangle, triangle, or irregular shape. 
The size is measured by the number of occluded pixels. The occlusion color specifies the colors occluded pixels can take. The coloring of an occlusion can be either uniform,
where all occluded pixels share the same color, or 
 multiform, where these colors can vary in the range of $[-\epsilon, \epsilon]$, where $\epsilon$ specifies the threshold between an occluded pixel's value and its original value.

 \begin{wrapfigure}{r}{0.3\textwidth}
	\vspace{-4ex}
    % \vspace{4mm}
	\includegraphics[width=0.28\textwidth]{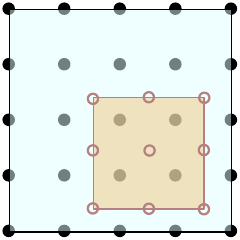}
	\vspace{-1mm}
	\caption{An example occlusion on a $5\times 5$ image at real number position.}
	\label{fig:occlusion interpolation}
	\vspace{-7mm}
\end{wrapfigure}

Prior studies \cite{eykholt2018robust,brown2017adversarial} showed that both the uniform and multiform occlusions could cause misclassification to neural networks. Fig. \ref{fig:occlusion example} shows two examples of multiform and uniform occlusions, respectively. The traffic sign for ``70km/h speed limit'' in Fig. \ref{fig:occlusion example}(a) is  misclassified to ``30km/h'' by adding a $5\times 5$ multiform occlusion. Fig. \ref{fig:occlusion example}(d) shows another sign, with different light conditions, 
where a $3\times 3$ uniform occlusion (in Fig. \ref{fig:occlusion example}(c)) causes the sign to be  misclassified to ``30km/h''.

The occlusion position is another aspect of defining occlusions. An occlusion can be placed precisely on the pixels of an image, or between a pixel and its neighbors. Fig. \ref{fig:occlusion interpolation} shows an example, where the dots represent image pixels and the circles are the occluding pixels that will substitute the occluded ones. We say that an occlusion pixel $\vartheta_{i',j'}$ at location $(i',j')$ surrounds an image pixel $p_{i,j}$ at location $(i,j)$ if and only if $|i-i'|<1$ and $|j-j'|<1$. Note that $i',j'$ are real numbers, representing the location where the occlusion pixel $o$ is placed on the image. An image pixel can be occluded by the substitute occlusion pixels if the occlusion pixels surround the image pixel. 

There are at most four surrounding occlusion pixels for each image pixel, as shown in Fig. \ref{fig:occlusion interpolation}. Let $\mathbb{I}_p$ be the set of the locations where  the surrounding occlusion pixels of $p$ are placed. After the occlusion, the value of pixel $p_{i,j}$ is altered to the new one denoted by $p'_{i,j}$, which can be computed by interpolation \cite{jaderberg2015spatial,Kirkland2010} such as next neighbour interpolation or Bi-linear interpolation based on occlusion pixels in $\mathbb{I}_p$. 
Besides that, we use a method based on $L_1$-distance to calculate how much a pixel is occluded. Since the $L_1$-distance of two adjacent pixels is 1, a surrounding occlusion pixel should not affect the image pixel if their $L_1$-distance is greater than $1$. The formula $max(0, (|1-i'+i|) + (1-j'+j) - 1)$ indicates how much an image pixel at $(i, j)$ is occluded by an occlusion pixel at $(i', j')$.  For instance, occlusion pixel at $(i', j')=(0.9, 0.9)$ has no effect to image pixel $(i, j)=(0, 0)$ since their $L_1$-distance is larger than 1. Therefore, the occlusion factor $s_{i, j}$ for pixel $p$ at $(i, j)$ can be calculated based on all surrounding occlusion pixels in $\mathbb{I}_p$ as:
\begin{align}
\label{eq:interpolation}
    s_{i, j}=max(0, \textstyle{\sum_{{i'_0, j'}\in \mathbb{I}_{p}}(|1-j+j'|)} + \textstyle{\sum_{i', j'_0\in \mathbb{I}_{p}}(|1-i'+i|)-1})
\end{align}
where $(i'_0, j'_0)$ is the first element of $\mathbb{I}_{p}$. Notably, $s$ is 1 for completely occluded pixel and 0 for the pixel that is not occluded, otherwise $s$ has a value between $(0, 1)$. Also, it is a special case for Equation \ref{eq:interpolation} when $(i', j')$ are integers, where $s$ can be reduced to $0$ or $1$.

%%% Local Variables:
%%% mode: latex
%%% TeX-master: "main"
%%% End:

\section{The Occlusion Robustness Verification Problem}
\label{sec:occrob}

Let $\mathbb{R}^{m\times n}$ be the set of images whose height is $m$ and width is $n$. We use $\mathbb{N}_{1, m}$ (\textit{resp.} $\mathbb{N}_{1, n}$) to denote the set of all the natural numbers ranging from $1$ to $m$ (\textit{resp.} $n$).
A coloring function $\zeta:\mathbb{R}^{m\times n}\times \mathbb{R} \times \mathbb{R} \to \mathbb{R}$ is a mapping of each pixel of an image to its corresponding color value. Given an image $x\in \mathbb{R}^{m\times n}$, $\zeta(x, i, j)$ defines the value to color the pixel of $x$ at $(i, j)$.

\begin{definition}[Occlusion function]
	\label{definition:occlusion function}
	Given a coloring function $\zeta$ and an occlusion $\vartheta$ of size $w\times h$ which is at position $(a, b)$, the occlusion function is defined as function $\gamma_{\zeta,w\times h}:\mathbb{R}^{m\times n}\times \mathbb{R}\times \mathbb{R}\rightarrow \mathbb{R}^{m\times n}$ such that $x'=\gamma_{\zeta,w\times h}(x,a,b)$ if for all $i\in \mathbb{N}_{1, n}$ and $j\in \mathbb{N}_{1, m}$, there is,  
    \begin{align}
    \label{eq:xij}
    &x'_{i, j}=x_{i, j} - s_{i, j}\times (x_{i, j} - \zeta(x, i, j)),\\
    \label{eq:zeta}
    \text{where}, &\ \zeta(x, i, j)=\frac{\sum_{(i', j')\in \mathbb{I}_{x_{i, j}}}\vartheta_{i',j'}\sqrt{(i-i')^2+(j-j')^2}}{\sum_{(i', j')\in \mathbb{I}_{x_{i, j}}}\sqrt{(i-i')^2+(j-j')^2}}.
    \end{align}
\end{definition}
\noindent $s$ in Equation \ref{eq:xij} is the occlusion factor for pixel at $(i, j)$ as mentioned in Sec. \ref{section:occlusion perturbation}. 
Note that when $i', j'$ are integers, Equation \ref{eq:xij} can be reduced to $x_{i, j}=\vartheta_{i, j}$, which represents that $x_{i,j}$ is completely occluded by the occlusion. In other words, the integer case is a special case of the real number case. Also, when pixel at $(i, j)$ is not occluded, since $s_{i,j}=0$. In this case, Equation \ref{eq:xij} can be reduced to $x'_{i, j} = x_{i, j}$.

Interpolation is handled by $\zeta$ showed in Equation \ref{eq:zeta}. It shows the standard form for the color of the new $x'_{i, j}$.
A unique color value is specified for all the pixels in the occluded area for a uniform occlusion. Therefore, $\zeta$ in Equation \ref{eq:zeta} can be reduced to $\zeta(x,i, j)=\mu$ for some $\mu \in[0,1]$. The coloring function in a multiform occlusion is defined as $\zeta(x,i,j) = x_{i, j} + \Delta_p$ with $\Delta_p\in [-\epsilon, \epsilon]$, where $\epsilon\in \mathbb{R}$ defines the threshold that a pixel can be altered.

\begin{definition}[Local occlusion robustness]
	\label{definition:occlusion-robustness}
	Given a DNN $F:\mathbb{R}^{m\times n}\rightarrow \mathbb{R}^r$, an occlusion function  $\gamma_{\zeta,w\times h}:\mathbb{R}^{m\times n}\times \mathbb{R}\times \mathbb{R}\rightarrow \mathbb{R}^{m\times n}$ with respect to  coloring function $\zeta$ and occlusion size $w\times h$, and an input image $x$, $F$ is called local occlusion robust on $x$ with $\gamma_{\zeta,w\times h}$ if  $\Phi_F(x)=\Phi_F(\gamma_{\zeta,w\times h}(x,a,b))$ holds for all $1\leq a\leq n$ and $1\leq b\leq m$. 
\end{definition}
Intuitively, Definition \ref{definition:occlusion-robustness} means that $F$ is robust on $x$ against the occlusions of $\gamma_{\zeta,w\times h}$, if on any occluded image of $x$ by the occlusion function $\gamma_{\zeta,w\times h}$, $F$ always returns the same classification result as on the original image $x$. Depending on the coloring function $\zeta$, the definition applies to various occlusions concerning shapes, colors, sizes, and positions.  We can also extend the above definition to the global occlusion robustness if $F$ is robust on all images concerning $\gamma_{\zeta,w\times h}$. 

We prove that even for the case of uniform occlusion, a special case of the multiform one,  the local occlusion robustness verification problem is NP-complete on the ReLU-based neural networks. We leave the details of the proof to Appendix \ref{appendix:proof-of-np-completeness}.

%%% Local Variables:
%%% mode: latex
%%% TeX-master: "main"
%%% End:

\section{SMT-Based Occlusion Robustness Verification}
\label{sec:encoding}
\vspace{-1mm}
\subsection{A Na{\"i}ve SMT Encoding Method}
% Based on the specific characteristics of occlusion perturbation, 
The verification problem of FNNs' local occlusion robustness can be straightforwardly encoded into an SMT problem. 
In Definition \ref{definition:occlusion-robustness}, 
we assume that $x$ is classified by $\Phi$ to the label $\ell_q$, i.e., $\Phi(x)=\ell_q$, for a label $\ell_q \in L$. 
To prove $F$ is robust on $x$ after $x$ is occluded by occlusion $\vartheta$ with size $w\times h$, it suffices to prove that $F$ classifies every occluded image $x'=\gamma_{\zeta, w\times h}(a,b)$ to $\ell_q$ for all $1\leq a\leq n$ and $1\leq b\leq m$. This is equivalent to proving that the following constraints are not satisfiable:
\begin{align} 
&1 \le a \le n, 1 \le b \le m,\\
&\begin{aligned}
\label{smt:occ}
&\textstyle\bigwedge_{i\in \mathbb{N}_{1, n}, j\in \mathbb{N}_{1, m}}\\ 
&\hspace{4mm} \left(((a-1 < i < a+w+1) \wedge (b-1 < j < b+h+1) \wedge x'_{i,j}=\gamma_{\zeta, w\times h}(x, a, b)_{i, j}) \vee\right.\\   
&\hspace{4mm} \left.((i \ge a+w+1) \vee (i\le a-1) \vee (j\ge b+h+1) \vee (j\le b-1))\wedge x'_{i,j}=x_{i,j})\right),
\end{aligned} \\   
&\textstyle\bigvee_{l\in \mathbb{N}_{1, q-1}\cup \mathbb{N}_{q+1, r}} F(x')_l\ge F(x')_{q}.\label{smt:label}  
\end{align}

The conjuncts in Eq.~\ref{smt:occ} define that $x'$ is an occluded instance of $x$, and the disjuncts in Eq.~\ref{smt:label} indicate that, when satisfiable,  there exists some label $\ell_i$ which has a higher probability than $\ell_q$ to be classified to. Namely, the occlusion robustness of $F$ on $x$ is falsified, with $x'$ being a witness of the non-robustness. Note that this naive encoding considers the occlusion position's real number cases since function $\gamma$ implicitly includes the interpolation.

\begin{figure}[t]
	\centering
	\includegraphics[width=\linewidth]{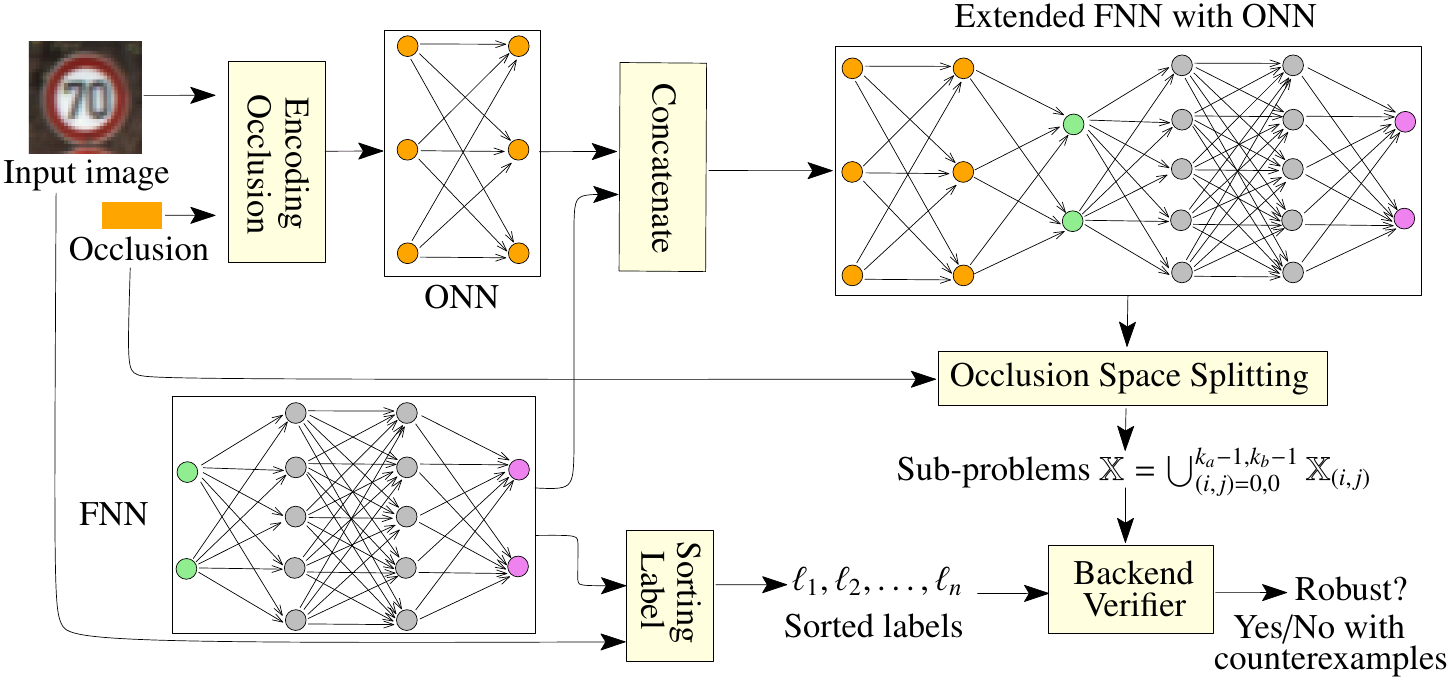}
	\vspace{-5mm}
	\caption{The workflow of encoding and verifying FNN's robustness against occlusions.}
	\label{fig:occlusion verification process}
	\vspace{-4mm}
\end{figure}

Although the above encoding is straightforward, solving the encoded constraints is experimentally beyond the reach of general-purpose existing SMT solvers due to the piece-wise linear ReLU activation functions in the definition of $F$ in the constraints of Eq.~\ref{smt:label}, and the large search space $m\times n\times {(2\epsilon)}^{w\times h}$ (see Experiment II in Sec. \ref{sec:exp}).

\vspace{-3mm}
\subsection{Our Encoding Approach}\label{subsec:encoding}
\vspace{-1mm}
\noindent\textbf{An Overview of the Approach.}
To improve efficiency, we propose a novel approach for encoding occlusion perturbations into four layers of neurons and concatenating the original network to these so-called \textit{occlusion layers}, constituting a new neural network which can be efficiently verified using state-of-the-art, SMT-based verifiers.

Fig. \ref{fig:occlusion verification process} shows the overview of our approach. 
Given an input image and an occlusion, 
we first construct a 3-hidden-layer occlusion neural network (ONN) and then concatenate it to the original FNN by connecting the ONN's output layer to the FNN's input layer. The combined network represents all possible occluded inputs and their classification results. The robustness of the constructed network can be verified using the existing SMT-based neural network verifiers.

We introduce two acceleration techniques to speed up the verification further. First, we divide the occlusion space into several smaller, orthogonal spaces, and verify a finite set of sub-problems on the smaller spaces. Second, we employ the eager falsification technique \cite{guo2021eager} to sort the labels according to their probabilities of being misclassified to. The one with a larger probability is verified earlier by the backend tools. Whenever a counterexample is returned, an occluded image is found such that its classification result differs from the original one. If all sub-problems are verified and no counterexamples are found, the network is verified robust on the input image against the provided occlusion.

\vspace{1mm}
\noindent\textbf{Encoding Occlusions as Neural Networks.} 
Given a coloring function $\zeta$, an occlusion size $w\times h$ and an input image $x$ of size $m\times n$, we construct a neural network $O:\mathbb{R}^{4+ct}\rightarrow \mathbb{R}^{m\times n}$ to encode all the possible occluded images of $x$, where $c=1$ if $x$ is a grey image and $c=3$ if $x$ is an RGB image, $t=0$ for the uniform occlusion and $t=w\times h$ for the multiform one. 

Fig. \ref{fig:occNN} shows the neural network architecture for encoding occlusions. We divide it into a fundamental part and an additional part. 
The former encodes the  occlusion position and the uniform occlusion color. 
The additional part is needed only by the multiform occlusion to encode the coloring function. 
Without loss of generality, we assume that 
the input layer takes the vector $(a,w,b,h,\zeta)$, where $(a,b)$ is the top-left coordinate of occlusion area in $x$. 
The coloring function $\zeta$ is admitted by other $c\times t$ neurons in the input layer when the occlusion is multiform.  

\begin{wrapfigure}{r}{0.58\textwidth}
	\vspace{-8mm}
	\includegraphics[width=\linewidth]{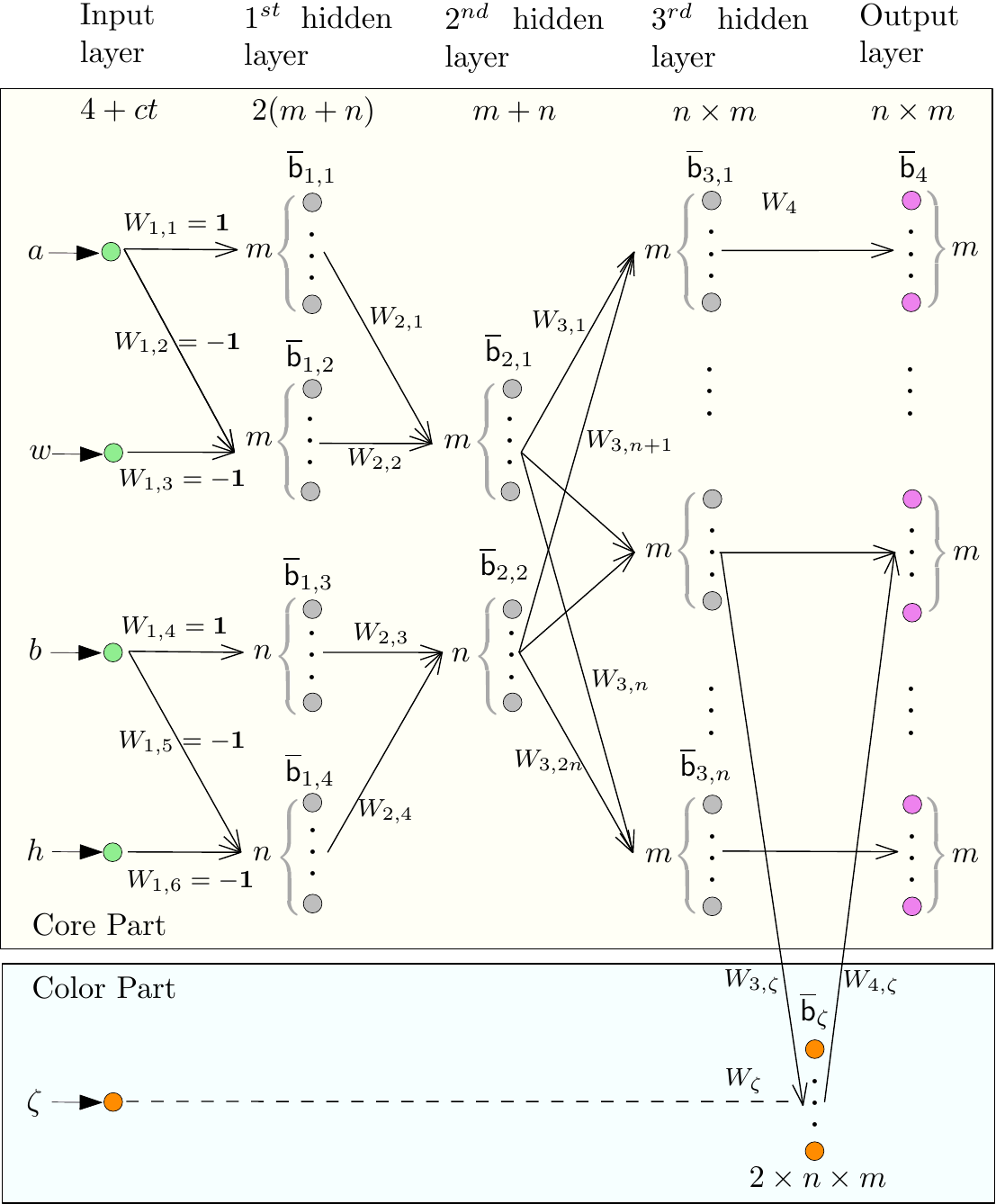}
	\caption{An occlusion neural network for the occlusions on an image $x$ with $\zeta$ and $w\times h$.}
	\label{fig:occNN}
	\vspace{-8mm}
\end{wrapfigure}

\vspace{1mm}
\noindent\textit{(1) Encoding occlusion positions.} We explain the weights and biases that are defined in the neural network to encode the occlusion position. On the connections between the input layer and the first hidden layer, the weights in matrices $W_{1,1}$, $W_{1,2}$ and $W_{1,3}$ are 1, -1 and -1, respectively. Note that we hide all the edges whose weights are 0 in the figure  for clarity.
The biases in $\overline{\textsf{b}}_{1, 1}$ are $(-1,-2,\ldots,-m)$ for the first $m$ neurons on the first hidden layer. 
Those in $\overline{\sf b}_{1, 2}$ are $(2,3,\ldots,m+1)$. 
The weights in $W_{1, 4}$, $W_{1, 5}$, $W_{1, 6}$ and the biases in $\overline{\textsf{b}}_{1, 3}$ and $\overline{\textsf{b}}_{1, 4}$ are defined in the same way. We omit the details due to the page limitation. 

For the second layer, the diagonals of weight matrices $W_{2, 1}$ to $W_{2, 4}$ are set to -1, and the rest of their entries are 0. The biases in $\overline{\sf b}_{2, 1}$ and $\overline{\sf b}_{2, 2}$ are $1$.  After the propagation to the second hidden layer, a pixel at position $(i,j)$ in the image $x$ is occluded if and only if both the outputs of the $i^{th}$ neuron in the first $m$ neurons and the $j^{th}$ neuron in the remaining $n$ neurons on the second hidden layer are 1. 

The third hidden layer represents the occlusion status of each pixel in the original image $x$. $2n$ weight matrices connect the second layer and the $n\times m$ neurons of the third layer. For example, we consider the weights in $W_{3, i}$ and $W_{3, n+i}$ which connect the $i^{th}$ group of $m$ neurons in the third layer to the second layer. The size of $W_{3, i}$ is $m \times m$, and the weights in the $i^{th}$ row are 1 while the rest is 0. The size of $W_{3, n+i}$ is $m\times n$. The weights on its diagonal are set to 1, while the rest are set to 0. All the biases in $\overline{\textsf{b}}_{3, 1}$ to $\overline{\textsf{b}}_{3, n}$ are -1. 
The output of the third layer indicates the occlusion status of all the pixels. If a pixel at $(i, j)$ is occluded, then the output of the $(i\times m + j)^{th}$ neuron in the third layer is 1, and otherwise, 0.

\vspace{1mm}
\noindent\textit{(2) Encoding Coloring Functions.} We consider the uniform and multiform coloring functions separately for verification efficiency, although the former is a special case of the latter. We first consider the general multiform case. 
In the multiform case, we introduce $2\times n\times m$ extra neurons in the third hidden layer, as shown in the bottom part of Fig. \ref{fig:occNN}. These neurons can be combined with the third layer, but it would be more clear to separate them. The weight matrix $W_{3, \zeta}$ connects the third layer to these neurons, with its first half of diagonal set to 1, and the second half set to -1. This helps retain the sign of the input $\zeta$ during propagation. 
The weight matrix $W_{\zeta}$ connects the input $\zeta$ to these neurons, whose diagonal are 1 and the biases $\overline{\textsf{b}}_{\zeta}$ are -1. 
These neurons works just like the third layer, except that they not only represent the occlusion status of pixels, but also preserve the input $\zeta$. If a pixel at $(i, j)$ is occluded and $\zeta$ has a positive value, then the $(i\times m + j)^{th}$ output in the first half of them is $\zeta$. The $(i\times m + j)^{th}$ output in the second half is $\zeta$ when $\zeta$ has a negative value. Otherwise, the output is 0. 
In the uniform case, it can be encoded together with input images, and we thus explain in the following paragraph.

\vspace{1mm}
\noindent\textit{(3) Encoding Input Images.} In the fourth layer, we use $W_4$ to denote the weight matrix connecting the third layer. $W_4$ is used to encode pixel values of the input image $x$ and the coloring function $\zeta$ of occlusions. In the uniform case, the weight $\textsf{w}(i, i)$ in the diagonal of $W_4$ is $\textsf{w}(i, i) = \zeta_i - x_i$ and the biases $\overline{\textsf{b}}_{4} = \textbf{x}$ where $\textbf{x}$ is the flattened vector of the original input image. 
In the multiform case, the weight matrix $W_{4, \zeta}$ connects the neurons in the bottom part that preserves information of input $\zeta$ to the fourth layer. The first half of $W_{4, \zeta}$ is identical to $W_4$, and the second half of $W_{4, \zeta}$ has its diagonal set to -1. It provides the value of the coloring function $\zeta$ with any sign for each occluded pixel. The output of the $j^{th}$ neuron in the $i^{th}$ group of the fourth layer is the raw pixel value plus $\zeta$ if the pixel at $(i, j)$ is occluded; otherwise, it is the raw pixel value of $p$.

\vspace{1mm}
\noindent\textbf{An Illustrative Example.} We show an example of constructing the occlusion network on a $2\times 2$, single-channel image in Fig. \ref{fig:occlusion layer example}. In this example, we assume that the input image is $x=[0.4, 0.6, 0.55, 0.72]$ and the occlusion applied to $x$ has a size of $1\times 1$, which means $w=1$ and $h=1$. For uniform occlusion, the coloring function $\zeta$
\begin{wrapfigure}{r}{0.6\textwidth}
	\vspace{-6mm}
	\includegraphics[width=\linewidth]{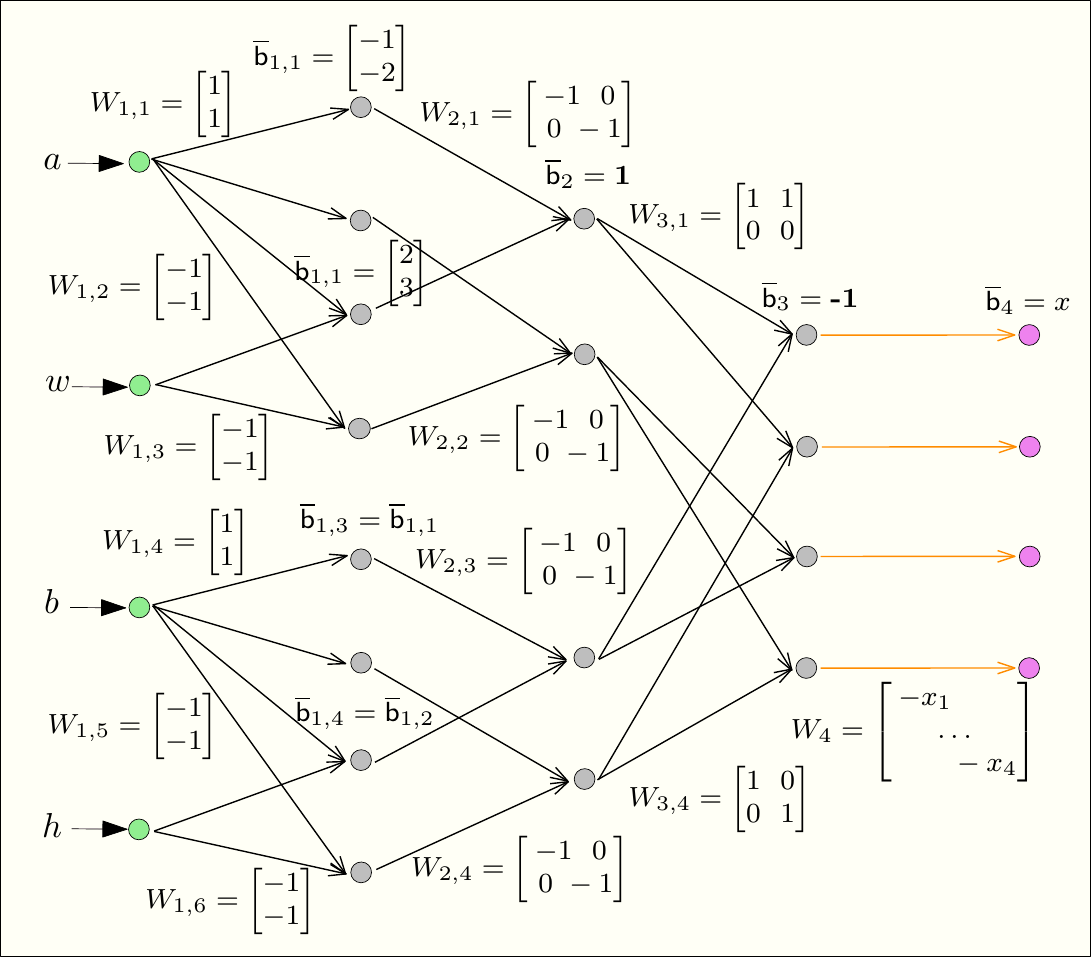}
	\vspace{-4mm}
	\caption{An example of encoding a one-pixel uniform occlusion as a neural network.}
	\vspace{-10mm}
	\label{fig:occlusion layer example}
\end{wrapfigure}
has a fixed value of 0, and for multiform case, the threshold $\epsilon$ that a pixel can be altered is set to $0.1$.

We suppose the occlusion is applied at position $(1, 2)$, which means $a=1$ and $b=2$ for the input of occlusion network. In the forward propagation, we calculate the output of the first layer by $a\times W_{1, 1} + \overline{\textsf{b}}_{1, 1}$ and $a\times W_{1, 2} + b\times W_{1, 3} + \overline{\textsf{b}}_{1, 2}$ and can get $(0, 0, 0, 1)$ for the first four neurons. Following the same process, we get the output of the second 4 neurons, $(1, 0, 0, 0)$. After propagation to the second layer, it outputs $(1, 0), (0, 1)$ based on $W_{2, 1}, W_{2, 2}$ and $\overline{\textsf{b}}_{2}$, representing the second column and the first row of $x$ is under occlusion. Likely, the third layer outputs $(0, 1, 0, 0)$ based on its weight matrices and biases, representing that the second pixel in the first row is occluded. 
After propagation to the fourth layer, the occlusion network outputs an occluded image $x'=[0.4, 0, 0.55, 0.72]$ based on $W_4$ and $\overline{\textsf{b}}_{4}$. It is identical to the expected occluded image, where the second pixel is occluded, and other pixels stay unchanged.
% In the fourth layer, the occluded image $x'$ is obtained based on the output of the third layer. Since the second pixel of $x$ is occluded and the color of the occlusion is 0, we have $x'=[0.4, 0, 0.55, 0.72]$.
Suppose we change $a$ to some real number, for instance, 1.5. After the same propagation, we will get an output of $(0, 0.5, 0, 0.5)$ in the third layer, representing that the neurons in the second column are affected by the occlusion by a factor of 0.5. The fourth layer then outputs $[0.4, 0.3, 0.55, 0.36]$, which is the corresponding occluded image $x'$.

In the multiform case, as mentioned at the first, we suppose the threshold $\epsilon=0.1$, and keep all other settings. Then after the same propagation to the third layer, the third layer will output $(0, 1, 0, 0)$, representing that the second pixel is occluded. Those extra neurons then output $(0, 0.1, 0, 0, 0, 0, 0, 0)$ where the second neuron in the first half is $0.1$ and 0 for the remaining. This indicates both that the second pixel in the first row is occluded, and has an epsilon of $0.1$. After propagation to the fourth layer, the occlusion network outputs $x’=[0.4, 0.7, 0.55, 0.72]$ based on its $W_4$ and $\overline{\textsf{b}}_{4}$. As expected, the second pixel is occluded and increases by $0.1$, and other pixels stay unchanged. For the case of a negative $\epsilon$ of $-0.1$, the extra neurons output $(0, 0, 0, 0, 0, 0.1, 0, 0)$. Note that the second neuron in the second half is $0.1$ and the remaining are 0, which helps retain the sign of $-0.1$. The fourth layer then outputs $[0.4, 0.5, 0.55, 0.72]$, which is the expected occluded image where the second pixel decreases by $0.1$.

\vspace{-2mm}
\subsection{The Correctness of the Encoding} 
\label{subsec:correctness of encoding}
\vspace{-1mm}
Given an input image $x$, a rectangle occlusion of size $w\times h$, and a coloring function $\zeta$, 
let $O$ be the corresponding occlusion neural network constructed in the approach above. Let $F$ be the FNN to verify. We concatenate $O$ to $F$ by connecting $O$'s output layer to $F$'s input layer. The combined network implements the composed function $F\circ O$. The problem of verifying the occlusion robustness of $F$ on the input image $x$ is reduced to a regular robustness verification problem of $F\circ O$. 

\begin{theorem}[Correctness]\label{thm:correctness}
An FNN $F$ is robust on the input image $x$ with respect to a rectangle occlusion in the size of $w\times h$ and a coloring function $\zeta$ if and only if $\Phi_{F\circ O}((a,w,b,h,\zeta))=\Phi_F(x)$ for all $1\leq a\leq n$ and $1\leq b\leq m$.
\end{theorem}

Theorem \ref{thm:correctness} means that all the occluded images from $x$ are classified by $F$ to the same label as $x$, which implies the correctness of our proposed encoding approach. To prove Theorem \ref{thm:correctness}, it suffices to show that the encoded occlusion neural network represents all the possible occluded images. In other words, when being perceived as a function, the network outputs the same occluded image as the occlusion function for the same occlusion coordinate $(a,b)$, as formalized in the following lemma. 

\begin{lemma}
	\label{definition:correctness of the encoding approach}
	Given an occlusion function $\gamma_{\zeta,w\times h}:\mathbb{R}^{m\times n}\times \mathbb{R}\times \mathbb{R}\rightarrow \mathbb{R}^{m\times n}$ and an input image $x$, 
	let $O_{\gamma,x}:\mathbb{R}^{4+ct}\rightarrow \mathbb{R}^{m\times n}$ be the corresponding occlusion neural network. There is $\gamma_{\zeta,w\times h}(x,a,b)=O_{\gamma,x}(a,w,b,h,\zeta)$ for all $1\leq a\leq n$ and $1\leq b\leq m$. 
\end{lemma}

\begin{proof}[Sketch]
	It suffices to prove $\gamma_{\zeta,w\times h}(x,a,b)_{i,j}=O_{\gamma,x}(a,w,b,h,\zeta)_{i,j}$ for all $i\in \mathbb{N}_{1, n}$ and $j\in \mathbb{N}_{1, m}$. By Definition \ref{definition:occlusion function}, we consider the following two cases:
	\proofpart{When a pixel $p$ at position $(i, j)$ is fully occluded, we have $\gamma_{\zeta,w\times h}(x,a,b)_{i,j}=\zeta (x,i,j)$. We need to prove that $O_{\gamma,x}(a,w,b,h,\zeta)_{i,j}=\zeta(x,i,j)$.} 
	
	Suppose $p$ is covered by an arbitrary uniform occlusion with size of $w_0\times h_0$ at position $(a_0, b_0)$. We can observe that for that pixel $p$,  $i > a_0 \wedge i < a_0 + w_0 - 1$ and $j > b_0 \wedge j < b_0 + h_0 - 1$ hold since $p$ is covered by the occlusion.

	We show the output of $O_{\gamma,x}(a,w,b,h,\zeta)_{i,j}$ by inspecting the $(i * n + j)^{th}$ output of the occlusion network after propagation, starting from inspecting the output of the $i^{th}$ and $(i+m)^{th}$ neurons of the first layer. According to the network structure discussed in Sec. \ref{subsec:encoding}, we can tell that the $i^{th}$ neuron in the first layer is 0 only when $i > a_0$, the same property holds for the $(i+m)^{th}$ neuron when $i < a_0 + w_0 - 1$. Therefore, the output for the $i^{th}$ and $(i+m)^{th}$ neurons of the first layer is 0, which leads to the $i^{th}$ neuron in the first part of the second layer has output of value 1. Through the similar process, we can get that the value of $z_j^{(2)}$ in the second part of the second layer is also 1. 
	
	The $(i \times n + j)^{th}$ neuron in the third layer is based on the $i^{th}$ neuron and $j^{th}$ neuron of the second layer that we just discussed. Therefore, the output of that neuron,  $z^{(3)}_{i \times n + j}$, is 1. 
	For uniform occlusion, suppose the coloring function $\zeta$ has a fixed value $\mu_0$. By propagating the output $z^{(3)}_{i \times n + j}$ to the fourth layer, which is calculated as $W_4 \times z^{(3)} + \overline{\textsf{b}}_{4}$, the $(i \times n + j)^{th}$ output of the fourth layer is $1\times (\mu_0 - p_{i, j}) + p_{i, j} = \mu_0$. Likely, for multiform occlusion, $\zeta$ indicates the threshold $\epsilon_0$ that a pixel can change. The $(i \times n + j)^{th}$ extra neuron outputs $\epsilon_0$ , then the corresponding neuron in the fourth layer outputs $p_{i, j} + \epsilon_0$.
	
	This output of $O_{\gamma,x}(a,w,b,h,\zeta)_{i,j}$ is identical to $\gamma_{\zeta,w\times h}(x,a,b)_{i,j}$, the expected pixel value at position $(i, j)$, which also indicates that the color is correctly encoded.

	\proofpart{When a pixel $p$ at position $(i, j)$ is not occluded, we have $\gamma_{\zeta,w\times h}(x,a,b)_{i,j}=x_{i,j}$. Then, we need to prove that $O_{\gamma,x}(a,w,b,h,\zeta)_{i,j}=x_{i,j}$.}
	
	% \item Case II. When $i<a\vee i\geq a+w$ and $j<b\vee j\geq b+h$, we have $\gamma_{\zeta,w\times h}(x,a,b)_{i,j}=x_{i,j}$. Then, we need to prove that $O_{\gamma,x}(a,w,b,h,\zeta)_{i,j}=x_{i,j}$.
	
	In this case, we can observe that $i<a_0\vee i\geq a_0+w_0$ and $j<b_0\vee j\geq b_0+h_0$ hold for pixel $p$. Then We can tell that the corresponding neuron in the third layer outputs 0 and  the output of the $(i * n + j)^{th}$ neuron in the fourth layer is the origin pixel value of $p$ following the similar process discussed in case 1.
	
\end{proof}

For the occlusion with real number position, some more cases need to be discussed, but the proof has a very similar sketch as the normal occlusion with integer position. 
We leverage the equality of $a\times b = exp(log(a)+log(b))$ and add it to the propagation between the third layer and those extra neurons only when the occlusion is at real number positions in the multiform case. 
And we use $ReLU(a+b-1)$ as an alternative to logarithms and exponents in implementation since Marabou does not support such operations.
A complete proof for Lemma \ref{definition:correctness of the encoding approach} is deferred to Appendix \ref{sec:lemma3}.

Theorem \ref{thm:correctness} can be directly derived from Lemma \ref{definition:correctness of the encoding approach} and Definition \ref{definition:occlusion-robustness} by substituting $\gamma_{\zeta,w\times h}(x,a,b)$ for $O_{\gamma,x}(a,w,b,h,\zeta)$ in the definition. 

\vspace{-2mm}
\subsection{Verification Acceleration Techniques}
\vspace{-1mm}
Existing SMT-based neural network verification tools can directly verify the composed neural network. 
The number of ReLU activation functions in the network is the primary factor in determining the verification time cost by the backend tools. In the occlusion part, the number of ReLU nodes is independent of the scale of the original networks to be verified. Therefore, our approach's scalability relies only on the underlying tools. 
%>>>>>>> ced80189df180092923344ea79983e335de9b6b2

To further improve the verification efficiency, we integrate two algorithmic acceleration techniques by dividing the verification problem into small independent sub-problems that can be solved separately. 

\vspace{1mm}
\noindent\textbf{Occlusion Space Splitting.} We observed that verifying the composed neural network with a large input space can significantly degrade the efficiency of backend verifiers. Even for small FNNs with only tens of ReLUs, the verifiers may run out of time due to the large occlusion space for searching. For instance, the complexity of Reluplex \cite{katz2017reluplex} can be derived from the original SMT method of Simplex \cite{nelder1965simplex}. It has a complexity of $\Omega(v\times m\times n)$, where $m$ and $n$ represent the number of constraints and variables, and $v$ represents the number of pivots operated in the Simplex method. In the worst case, $v$ can grow exponentially. Reduction in the search space can reduce the number of pivot operations, therefore significantly improving verification efficiency.

Based on the above observation, we can divide $[1,m]$ (\textit{resp.} $[1,n]$) into $k_m\in \mathbb{Z}^{+}$ (\textit{resp.} $k_n\in \mathbb{Z}^{+}$) intervals $[m_0,m_1],\ldots,[m_{k_m-1},m_{k_m}]$ (\textit{resp.} $[n_0,n_1],\ldots,[n_{k_n-1},n_{k_n}]$) and verify the problem on the Cartesian product of the two sets of intervals. 
%n most cases, the verification task can not be completed before time is used up. We observed that:
\begin{equation}
	\begin{aligned}
		&\forall x' \in \mathbb{X}.\Phi(x') = \Phi(x)\equiv \textstyle\bigwedge^{(k_m-1,k_n-1)}_{(i,j)=(0,0)} \forall x' \in \mathbb{X}_{(i,j)}.\Phi(x') = \Phi(x),\ \text{where}\\ 
&\mathbb{X}=\textstyle\bigcup_{(i,j)=(0,0)}^{(k_m-1,k_n-1)}\mathbb{X}_{(i,j)}=\textstyle\bigcup_{(i,j)=(0,0)}^{(k_m-1,k_n-1)}\{\gamma_{\zeta, w\times h}(x, a, b)|m_{i}\leq a\leq m_{i+1},n_{j}\leq b\leq n_{j+1}\}.\\
	\end{aligned}
\end{equation}
%<<<<<<< HEAD
In this way, we split the occlusion space into $k_m\times k_n$ sub-spaces. It is equivalent to prove $\forall x' \in \mathbb{X}.\Phi(x')$ for all $\mathbb{X}_{(i,j)}$ with $0\leq i<k_m$ and $0\leq j<k_n$, without losing the soundness and completeness. We call each verification instance a \emph{query}, which can be solved more efficiently than the one on the whole occlusion space by backend verifiers. Furthermore, another advantage of occlusion space splitting is that these divided queries can be solved in parallel by leveraging multi-threaded computing.  

\vspace{2mm}
\noindent\textbf{Eager Falsification by Label Sorting.} 
Another \textit{Divide \& Conquer} approach for acceleration is to divide the verification problem into independent sub-problems by the classification labels in $L$, as defined below: 
\begin{equation}
	\begin{aligned}
		&\forall x' \in \mathbb{X}.\Phi(x') = \Phi(x)\equiv \forall x' \in \mathbb{X}.\textstyle\bigwedge_{\ell'\in {L}}\Phi(x) = \ell'\vee \Phi(x') \neq \ell'.\label{eq:sort}
	\end{aligned}
\end{equation}
The dual problem to disprove the robustness can be solved to find some label $\ell'$ such that $\Phi(x) \neq \ell'\wedge \Phi(x') = \ell'$. 
We can first solve those that have higher probabilities of being non-robust. Once a sub-problem is proved non-robust, the verification terminates, with no need to solve the remainder. 
Such approach is called \emph{eager falsification} \cite{guo2021eager}. Based on this methodology, we sort the sub-problems in a descent order according to the probabilities at which the original image is classified to the corresponding labels by the neural network. A higher probability implies that the image is more likely to be classified to the corresponding label. Heuristically, there is a higher probability of finding an occlusion such that the occluded image is misclassified to that label. We sequence the queries into backend verifiers until all are verified, or a non-robust case is reported. Our experimental results will show that this approach can achieve up to 8 and 24 times speedup in the robust and non-robust cases, respectively.

%%% Local Variables:
%%% mode: latex
%%% TeX-master: "main"
%%% End:

\section{Implementation and Evaluation}
\label{sec:exp}

We implemented our approach in a Python tool called \textsc{OccRob}, using the PyTorch framework. As a backend tool, we chose the Marabou \cite{katz2019marabou} state-of-the-art, SMT-based DNN verifier.
We evaluated our proposed approach extensively on a suite of benchmark datasets, including MNIST \cite{lecun-mnisthandwrittendigit-2010} and GTSRB \cite{Houben-IJCNN-2013}. The size of the networks trained on the datasets for verification is measured by the number of ReLUs, ranging from 70 to 1300.  All the experiments are conducted on a workstation equipped with a 32-core AMD Ryzen Threadripper CPU @ 3.7GHz and 128 GB RAM and Ubuntu 18.04. We set a timeout threshold of 60 seconds for a single verification task. All code and experimental data, including the models and verification scripts can be accessed at \url{https://github.com/MakiseGuo/OccRob}.

We evaluate our proposed method concerning efficiency and scalability in the occlusion robustness verification of ReLU-based FNNs. Our goals are threefold:

\begin{enumerate}
	\item To demonstrate the effectiveness of the proposed  approach for the robustness verification against various types of occlusion perturbations.
	\item 
	To evaluate the efficiency improvement of the proposed approach, compared with the naive SMT-based method.
	\item 
	To demonstrate the effectiveness of the acceleration techniques in efficiency improvement. 
\end{enumerate}

\begin{table}[t]
	\centering
	\caption{Occlusion verification results on two medium FNNs trained on MNIST and GTSRB in different occlusion sizes $2\times 2$ and $5\times 5$ and occlusion radius $\epsilon$.}
	\label{table:Experiment 1 result}
	\setlength{\tabcolsep}{1mm}
	\footnotesize 
	\begin{threeparttable}
		\begin{tabular}{|l|c|R{9.5mm}|rrrr|R{9.5mm}|rrrr|} 
			\hline 
			&  \multicolumn{1}{c|}{}  & \multicolumn{5}{c|}{Medium FNN (600 ReLUs) on MNIST}  & \multicolumn{5}{c|}{Medium FNN (343 ReLUs) on GTSRB} \\ 
			\hline
			Size & \multicolumn{1}{c|}{$\epsilon$} & \multicolumn{1}{c|}{- / +} & \multicolumn{1}{c|}{$T_{+}$} & \multicolumn{1}{c|}{$T_{-}$} & \multicolumn{1}{c|}{$T_{\text{build}}$} &\multicolumn{1}{c|}{TO(\%)} & \multicolumn{1}{c|}{- / +} & \multicolumn{1}{c|}{$T_{+}$} & \multicolumn{1}{c|}{$T_{-}$} & \multicolumn{1}{c|}{$T_{\text{build}}$} &\multicolumn{1}{c|}{TO(\%)}\\ 
			\hline
			\multirow{5}{*}{$2\times 2$} 			
			&0.05  & \textbf{2} / 28  & 120.01  & 11.98 & 0.068  & 0.00 & \textbf{8} / 13  & 103.64  & 24.18 & 0.089  & 0.00  \\
			&0.10  & \textbf{3} / 27  & 121.37  & 19.18 & 0.067 & 0.00 & \textbf{8} / 13  & 108.62  & 22.57 & 0.088  & 0.00  \\
			&0.20  & \textbf{4} / 26  & 122.12  & 39.57 & 0.067  & 0.00 & \textbf{10} / 11  & 113.7  & 23.17 & 0.084  & 0.00  \\
			&0.30  & \textbf{4} / 26  & 122.74  & 39.85 & 0.068  & 0.00 & \textbf{11} / 10  & 117.97  & 26.41 & 0.089  & 0.00  \\
			&0.40  & \textbf{6} / 24  & 126.66  & 49.6 & 0.074  & 0.00 & \textbf{14} / 7  & 115.49  & 31.53 & 0.096  & 0.14  \\
			\hline
			\multirow{5}{*}{$5\times 5$} 
			&0.05  & \textbf{5} / 25  & 123.45 & 49.04 & 0.065  & 0.00 & \textbf{9} / 12  & 123.99  & 26.02 & 0.101  & 0.00  \\
			&0.10  & \textbf{6} / 24  & 124.13 & 44.09 & 0.073 & 0.00 & \textbf{12} / 9  & 127.65  & 26.96 & 0.01  & 0.00  \\
			&0.20  & \textbf{10} / 20  & 179.89  & 52.51 & 0.073  & 3.26 & \textbf{16} / 5  & 126.98  & 27.22 & 0.102  & 0.00  \\
			&0.30  & \textbf{14} / 16  & 284.67  & 65.98 & 0.076  & 5.45\ & \textbf{18} / 3  & 146.68  & 29.11 & 0.100  & 0.04  \\
			&0.40  & \textbf{22} / 8  & 339.78  & 97.28 & 0.074  & 7.33 & \textbf{19} / 2  & 169.17  & 26.52 & 0.103  & 0.09  \\
			\hline 			
		\end{tabular}
		\begin{tablenotes}
			\footnotesize
			\item[*] - / +: the numbers of non-robust and robust cases; $T_{+}$ (\textit{resp.} $T_{-}$): average verification time in robust (\textit{resp.} non-robust) cases; $T_{\text{build}}$: the building time of occlusion neural networks; TO (\%): the percentage of runtime-out cases among all the queries. 
		\end{tablenotes}
	\end{threeparttable}
	\vspace{-5mm}
\end{table}

\noindent\textbf{Experiment I: Effectiveness.}  We first evaluate the effectiveness of \occrob in robustness verification against various types of occlusions of different sizes and color ranges. Table \ref{table:Experiment 1 result} shows the verification results and time costs against multiform occlusions on two medium FNNs trained on MNIST and GTSRB. We consider two occlusions sizes, $2\times 2$ and $5\times 5$, respectively. 
The occluding color range is from 0.05 to 0.40. In each verification task, we selected the first 30 images from each of the two datasets and verified the network's robustness around them, under corresponding occlusion settings. As expected, larger occlusion sizes and occluding color ranges imply more non-robust cases. 
One can see that \occrob can almost always verify and falsify each input image, except for a few time-outs. The robust cases cost more time than the non-robust cases, but all can be finished in a few minutes. Note that the time overhead for building occlusion neural networks is almost negligible, compared with the verification time. 
The effectiveness against uniform occlusions is shown in the following experiment.

\begin{figure}[t]
\includegraphics[width=\textwidth]{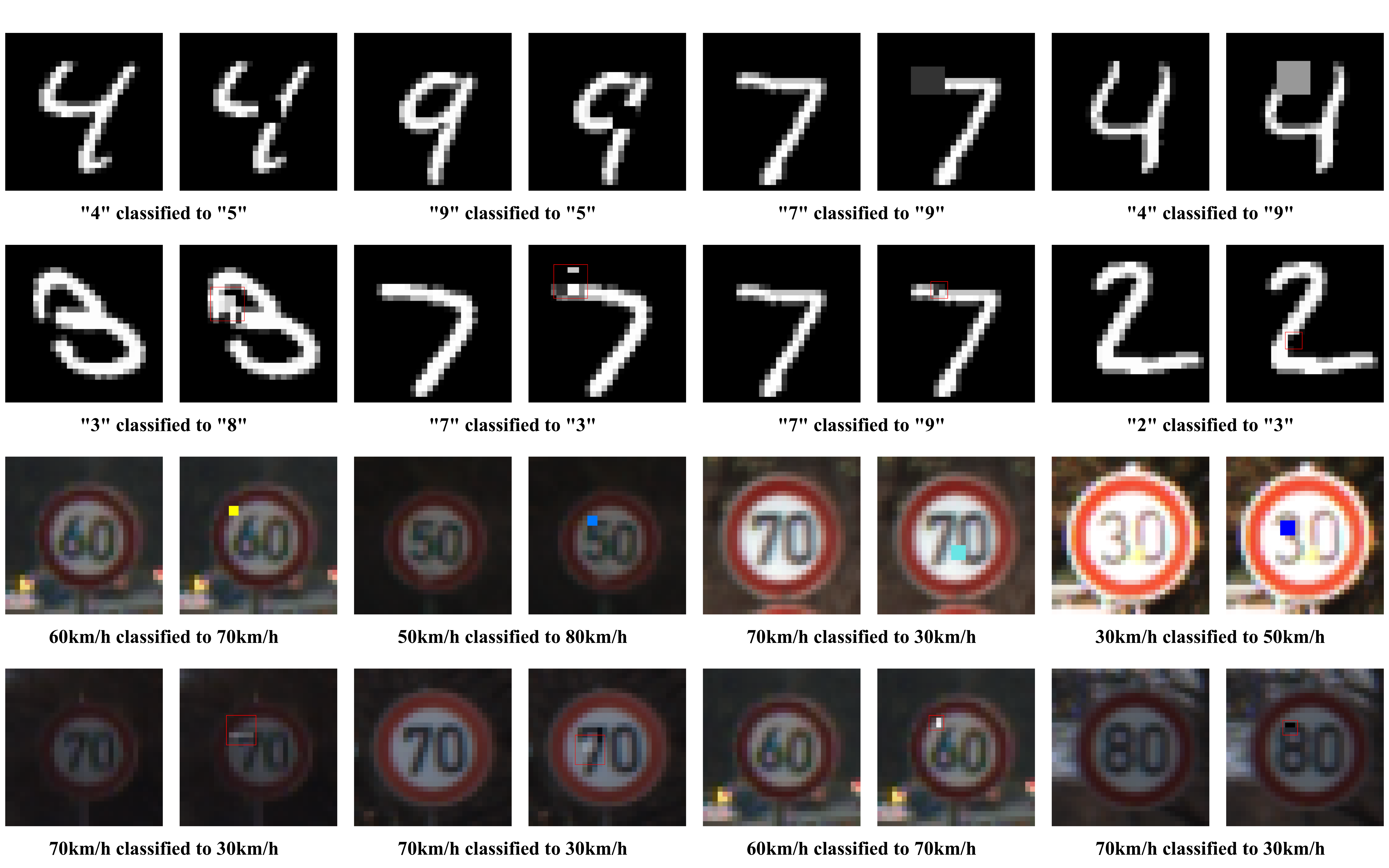}
\caption{Occlusive adversarial examples automatically generated for non-robust images.}
\label{fig:adv}
\vspace{-7mm}
\end{figure}

Fig. \ref{fig:adv} shows several occlusive adversarial examples that are generated by \occrob under different occlusion settings. 
These occlusions do not alter the semantics of the original images and should be classified to the same results as those non-occluded ones. However, they are misclassified to other results.

\vspace{1mm}
\noindent\textbf{Experiment II: Efficiency improvement over the naive encoding method.} We compare the efficiency of \occrob with that of a naive SMT encoding approach on verifying uniform occlusions since the naive encoding approach cannot handle verification against multiform occlusions.
We apply the same acceleration techniques, such as parallelization and a variant of input space splitting, to the naive approach, which otherwise times out for almost all verification tasks even on the smallest model.

Table \ref{table:expII} shows the average verification time on six FNNs of different sizes against uniform occlusions. We can observe that \occrob affords a significant improvement in efficiency, up to 30 times higher than the naive approach. 
It can always finish before the preset time threshold, while the naive method fails to verify the two large networks under the same time threshold. The timeout proportion of two medium networks is over 70\%. While the small network on MNIST only has an 8\% of timeout proportion with the naive method, \occrob barely timeouts on every network(see Appendix \ref{appendix:full data of experiment 2}).

\begin{table}[h!]
\vspace{-4mm}
	\caption{Performance comparison between \occrob (OR) and the naive (NAI) methods on MNIST and GTSRB under different occlusion sizes.}
	\label{table:expII}
	\setlength{\tabcolsep}{0.6mm}
	\footnotesize  
	\begin{tabular}{|c|rr|rr|rc|rr|rr|rc|} 
		\hline 
		 & \multicolumn{6}{c|}{MNIST} & \multicolumn{6}{c|}{GTSRB}\\\hline 
		FNNs& \multicolumn{2}{c|}{Small FNN} & \multicolumn{2}{c|}{Medium FNN} & \multicolumn{2}{c|}{Large FNN} & \multicolumn{2}{c|}{Small FNN} & \multicolumn{2}{c|}{Medium FNN} & \multicolumn{2}{c|}{Large FNN} \\
		\hline  
		Size & OR~~ &  NAI~~~ & OR~~~ & NAI~~ & OR~~ & NAI & OR~~ & NAI~~ & OR~~ & NAI~~ & OR~~ & NAI  \\\hline 
		$1\times 1$ & 46.44 & 63.12  & 110.18  & 759.93 & 206.50  & TO  & 29.76 & 472.23 & 69.28& 989.08 & 173.62 & TO\\ 
		$2\times 2$ & 49.62 & 165.53 & 98.60 & 832.98 & 199.17& TO  & 21.04 & 340.89 & 42.16  & 680.81 & 103.42& TO\\ 
		$3\times 3$ & 51.23 & 298.59 & 111.14& 863.74 & 205.67& TO  & 11.93 & 169.35 & 32.00  & 499.31 & 81.17& TO\\ 
		$4\times 4$ & 44.78 & 256.22 & 115.99& 886.73 & 225.02& TO  & 8.90  & 141.85 & 31.24& 419.62 & 106.41& TO\\ 
		$5\times 5$ & 48.96 & 270.23 & 113.01& 803.40 & 264.79& TO  & 6.11 & 190.81 & 27.97& 418.56 & 118.99& TO\\ 
		$6\times 6$ & 47.81 & 318.28 & 127.98& 642.01 & 288.18& TO  & 7.49 & 213.35 & 21.70& 282.04 & 60.02& TO\\ 
		$7\times 7$ & 34.99 & 357.78 & 124.47& 589.41 & 222.65& TO  & 6.02 & 153.81 & 31.96& 404.18 & 62.60& TO\\ 
		$8\times 8$ & 36.05 & 324.34 & 129.27& 469.24 & 215.53& TO  & 5.99 & 123.07 & 28.44& 250.97 & 54.37& TO\\ 
		$9\times 9$ & 34.58 & 224.01 & 141.54& 375.97 & 219.61& TO  & 6.42 & 102.39 & 31.30& 160.84 & 59.87& TO\\ 
		$10\times 10$& 28.98 & 178.44 & 78.89 & 398.01 & 182.36  & TO& 6.61 & 127.20 & 28.59 &  153.96& 40.69& TO \\
		\hline 
	\end{tabular}
\vspace{-3mm}
\end{table}

\vspace{1mm}
\noindent\textbf{Experiment III: Effectiveness of the integrated acceleration techniques.}
We finally evaluate the effectiveness of the two acceleration techniques integrated with the tool. 
We evaluate each technique separately by excluding it from \occrob and comparing the verification time of \occrob and the corresponding excluded versions. 
Fig.~\ref{fig:result of experiment 3} shows the experimental results of verifying the medium FNN trained on GTSRB against multiform occlusions by the tools. Fig.~\ref{fig:result of experiment 3} (a) shows that label sorting can improve efficiency in both robust and non-robust cases. In particular, the improvement is more significant in the non-robust case, with up to 5 times speedup in the experiment. That is because solving each query is faster than solving all simultaneously, and further \occrob immediately stops dispatching queries once a counterexample is found in the non-robust case. 
Fig.~\ref{fig:result of experiment 3} (b) shows that occlusion space splitting can also significantly improve the efficiency, with up to 8 and 24 times speedups in the robust and non-robust cases, respectively. In addition, 
Fig.~\ref{fig:result of experiment 3} (b) also shows a significant reduction in the number of time-outs.

\begin{figure}[h!]
	\vspace{-6mm}
	\centering
	\includegraphics[width=\textwidth]{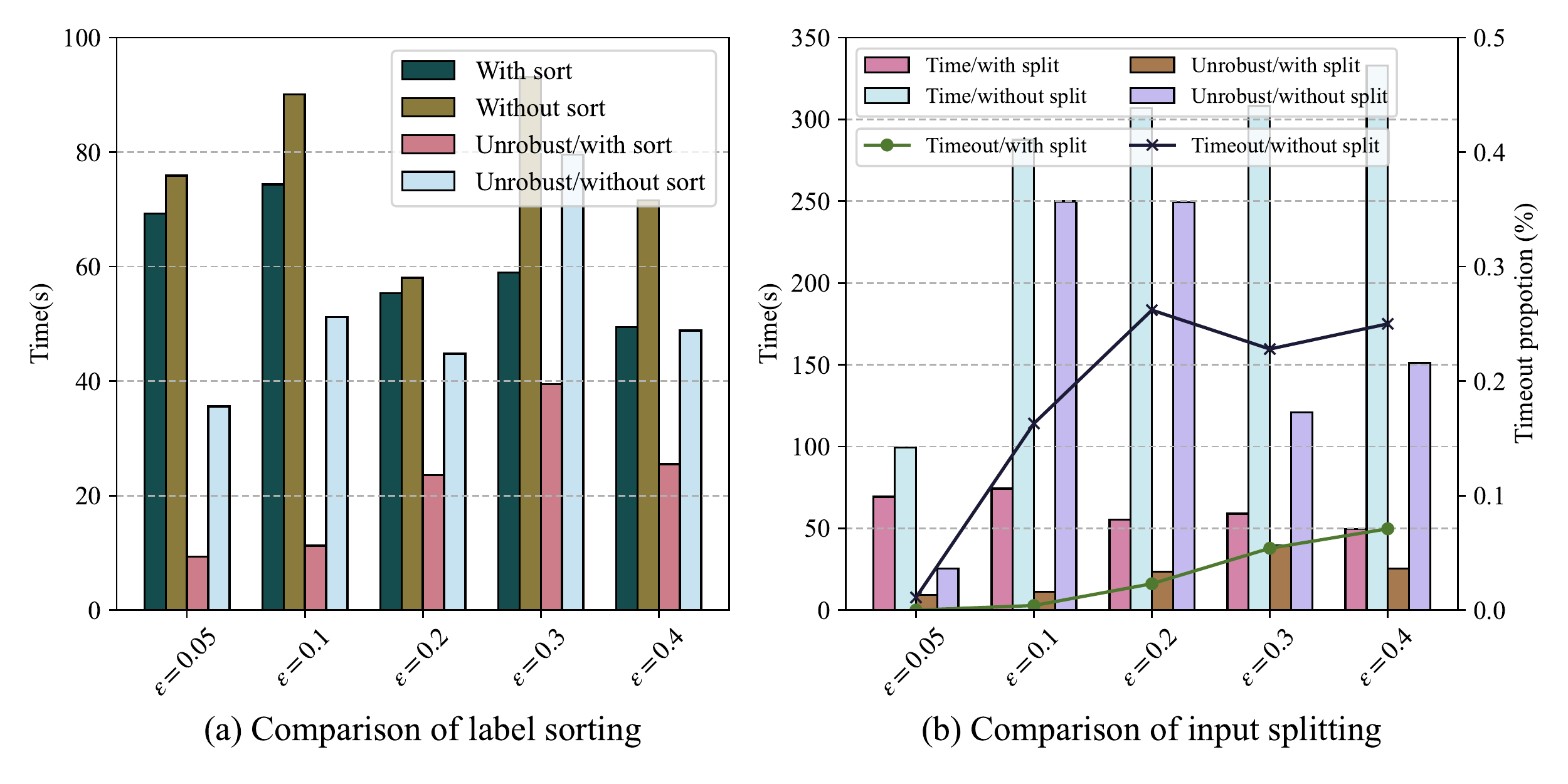}
	\vspace{-8mm}
	\caption{Efficiency evaluation results of the two acceleration techniques.}
	\vspace{-8mm}
	\label{fig:result of experiment 3}
\end{figure}

%%% Local Variables:
%%% mode: latex
%%% TeX-master: "main"
%%% End:

\vspace{-2mm}
\section{Related Work}
\vspace{-1mm}
\label{sec:rel}
Robustness verification of neural networks has been extensively studied recently, aiming at devising efficient methods for verifying neural networks' robustness against various types of perturbations and adversarial attacks. We classify those methods into two categories according to the type of perturbations, which can be semantic or non-semantic. Semantic perturbation has an interpretable meaning, such as occlusions and geometric transformations like rotation, while non-semantic perturbation means that noises perturb inputs with no particular meanings. 

Non-semantic perturbations are usually represented as $L_p$ norms, which define the ranges in which an input can be altered. Some robustness verification approaches for non-semantic perturbations are both sound and complete by leveraging SMT \cite{katz2017reluplex,amir2021smt} and MILP (mixed integer linear programming) \cite{singh2018robustness} techniques, while some sacrifice the completeness for better scalability by over-approximation \cite{lyu2020fastened,boopathy2019cnn,elboher2020abstraction}, abstract interpretation \cite{raghunathan2018certified,gehr2018ai2,cohen2019certified}, interval analysis by symbolic propagation  \cite{wang2018formal,NEURIPS2018_2ecd2bd9,li2019analyzing}, etc.

In contrast to a large number of works on non-semantic robustness verification, there are only a few studies on the semantic case. Because semantic perturbations are beyond the range of $L_p$ norms \cite{fischer2020certified}, those abstraction-based approaches cannot be directly applied to verifying semantic perturbations. Mohapatra et al. \cite{mohapatra2020towards} proposed to verify neural networks against semantic perturbations by encoding them into neural networks. Their encoding approach is general to a family of semantic perturbations such as brightness and contrast changes and rotations. Their approach for verifying occlusions is restricted to uniform occlusions at integer locations. Sallami et al.\cite{mziou2019safety} proposed an interval-based method to verify the robustness against the occlusion perturbation problem under the same restriction.
Singh et al.~\cite{singh2019abstract} proposed a new abstract domain to encode both non-semantic and semantic perturbations such as rotations.  
Chiang et al.~\cite{chiang2020certified} called occlusions \emph{adversarial patches} and proposed a certifiable defense by extending interval bound propagation (IBP) \cite{gowal2018effectiveness}. Compared with these existing verification approaches for semantic perturbations, our SMT-based approach is both sound and complete, and meanwhile, it supports a larger class of occlusion perturbations.

\vspace{-2mm}
\section{Conclusion and Future Work}
\label{sec:con}
\vspace{-1mm}
We introduced an SMT-based approach for verifying the robustness of deep neural networks against various types of occlusions. An efficient encoding method was proposed to represent occlusions using neural networks, by which we reduced the occlusion robustness verification problem to a regular robustness verification problem of neural networks and leveraged \emph{off-the-shelf} SMT-based verifiers for the verification. We implemented a resulting prototype \textsc{OccRob} and intensively evaluated its effectiveness and efficiency on a series of neural networks trained on the public benchmarks, including MNIST and GTSRB. 
Moreover, as the scalability of DNN verification engines continues to improve, our approach, which uses them as blackbox backends, will also become more scalable.

As our occlusion encoding approach is independent of target neural networks, we believe it can be easily extended to other complex network structures, such as convolutional and recurrent ones, which only depend on the backend verifiers. 
It would also be interesting to investigate how the generated adversarial examples could be used for neural network repairing \cite{usman2021nn,islam2020repairing} to train more robust networks.

\vspace{-2mm}
\section*{Acknowledgments} 
\vspace{-1mm}
 This work has been supported by National Key Research Program (2020AAA0107800), NSFC-ISF Joint Program (62161146001, 3420/21) and NSFC projects (61872146, 61872144), Shanghai  Science and Technology Commission (20DZ1100300), Shanghai Trusted Industry Internet Software Collaborative Innovation Center and ``Digital Silk Road" Shanghai International Joint Lab of Trustworthy Intelligent Software (Grant No. 22510750100).

%%% Local Variables:
%%% mode: latex
%%% TeX-master: "main"
%%% End:

\bibliographystyle{splncs04}
\bibliography{tacas2023}
\clearpage 

\appendix
\renewcommand{\thelemma}{\arabic{lemma}}

\section{Proof for NP-completeness}
\label{appendix:proof-of-np-completeness}

\subsection{Claim}
Let $F:\mathbb{R}^{m\times n}\rightarrow\mathbb{R}^r$ be a ReLU-based FNN and let $\gamma_{\zeta,\ w\times\ h}\left(x,a,b\right)$ be the occlusion function, where $x$ is an $m\times\ n$ image, $\left(a,b\right)$ is the top-left point coordinate of the occlusion rectangle, $w$ and $h$ is the width and height of occlusion rectangle, and $\zeta$ is coloring function. We say that the ReLU-based FNN $F$ is locally occlusion robust on $x$ with  $\gamma_{\zeta,\ w\times\ h}$ if for all $1\le a\le n\ and\ 1\le b\le m$, there is $\mathrm{\Phi}\left(F\left(x\right)\right)=\mathrm{\Phi}\left(F\left(\gamma_{\zeta,\ w\times\ h}\left(x,a,b\right)\right)\right)$.

\claim The problem of determining whether a ReLU-based FNN $F$ is locally occlusion robust on input image $x$ w.r.t. an occlusion function $\gamma_{\zeta,\ w\times\ h}$ is NP-complete.

\vspace{-2mm}

\subsection{Proof}
\vspace{-1mm}
Following the proof scheme for NP-complete problems in \cite{kleinberg2006algorithm}, to prove the NP-completeness of the problem of verifying the local occlusion robustness of ReLU-based FNNs, it is necessary to show that the problem is in NP and is NP-hard.

\subsubsection{NP-membership}
\vspace{-1mm}

\begin{lemma}
\label{lemma:np-membership}
    The problem of verifying the local occlusion robustness of ReLU-based FNNs is in NP.
\end{lemma}

\begin{proof}
We show that the problem is in NP. First, we assign $a,b$ as the top-left coordinate of the occlusion rectangle and $\zeta$ as the coloring function. Next, we get input assignment $x'=\gamma_{\zeta,\ w\times\ h}\left(x,a,b\right)$, which is polynomial in the size of $x$ by processing one pixel after another. Then, we feed the input matrix $x$ and $x'$ forward through the network $F$ respectively, which is polynomial in the size of the network $F$\cite{katz2017reluplex,salzer2021reachability}. Finally, we check whether $ \mathrm{\Phi}\left(F\left(x^\prime\right)\right)=\mathrm{\Phi}\left(F\left(x\right)\right)$ holds, which is polynomial in the size of $F\left(x\right)$ or $F\left(x'\right)$.

Therefore, we can check whether $x'=\gamma_{\zeta,\ w\times\ h}$ and $x$ have the same classified result under an assignment $a$,$b$ and $\zeta$ in polynomial time.
\end{proof}

\subsubsection{NP-hardness}
\ 

It takes two steps to prove the NP-hardness of the problem.

First, we find the problem of verifying the local robustness of ReLU-based OMNNs as a known NP-complete problem by explaining that it is an NNReach problem\cite{salzer2021reachability}.

Secondly, we reduce the problem of verifying the local robustness of ReLU-based OMNNs to the problem of verifying the local occlusion robustness of ReLU-based FNNs in polynomial time.

\begin{figure}[htbp]
\begin{minipage}[t]{0.4\linewidth}
\centering
\includegraphics[scale=0.35, trim=340 120 330 95, clip]{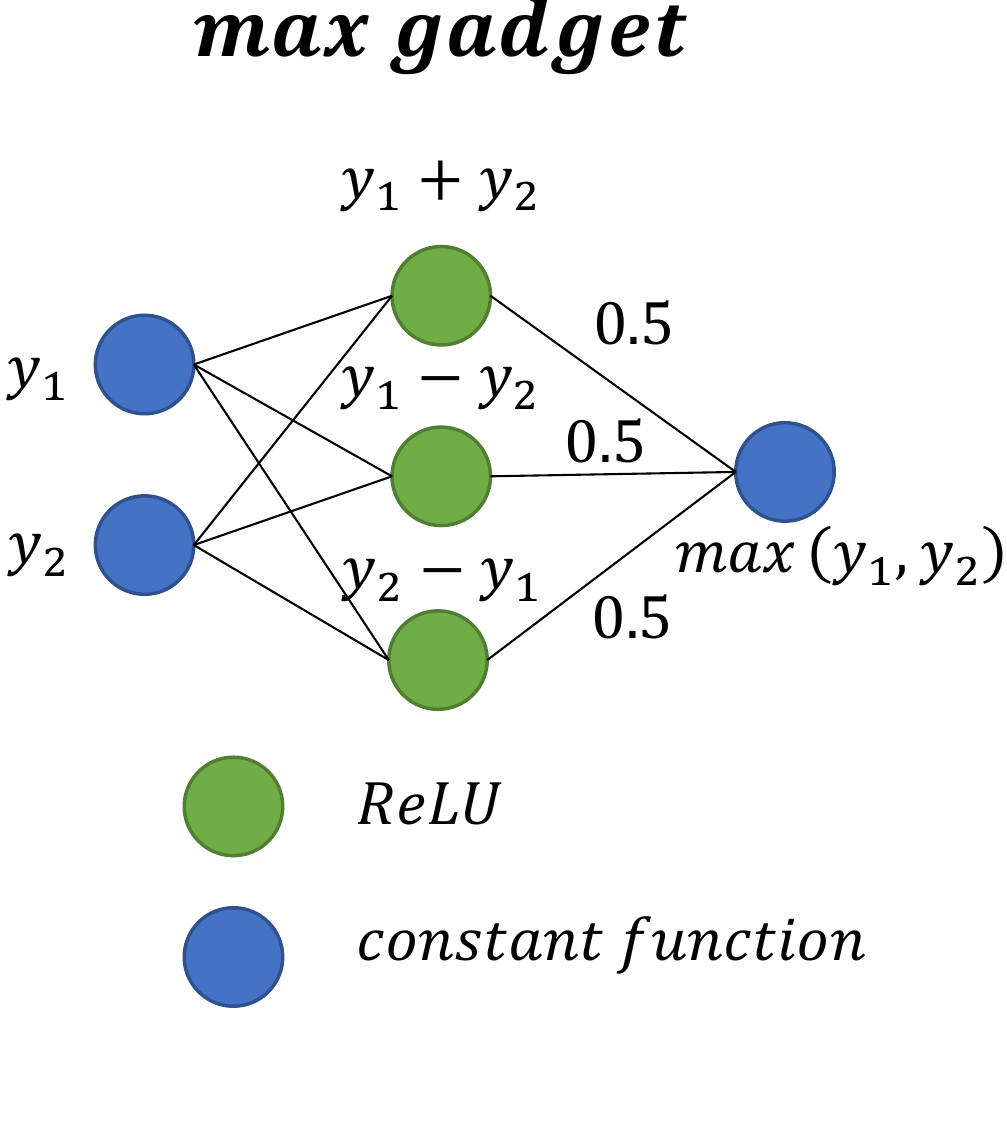}
% \vspace{-3mm}
\caption{The max gadget on ReLUs}
% \vspace{-2mm}
\label{fig:max-gadget}
\end{minipage}
\hfill
\begin{minipage}[t]{0.6\linewidth}
\centering
\includegraphics[scale=0.35, trim=270 70 150 120, clip]{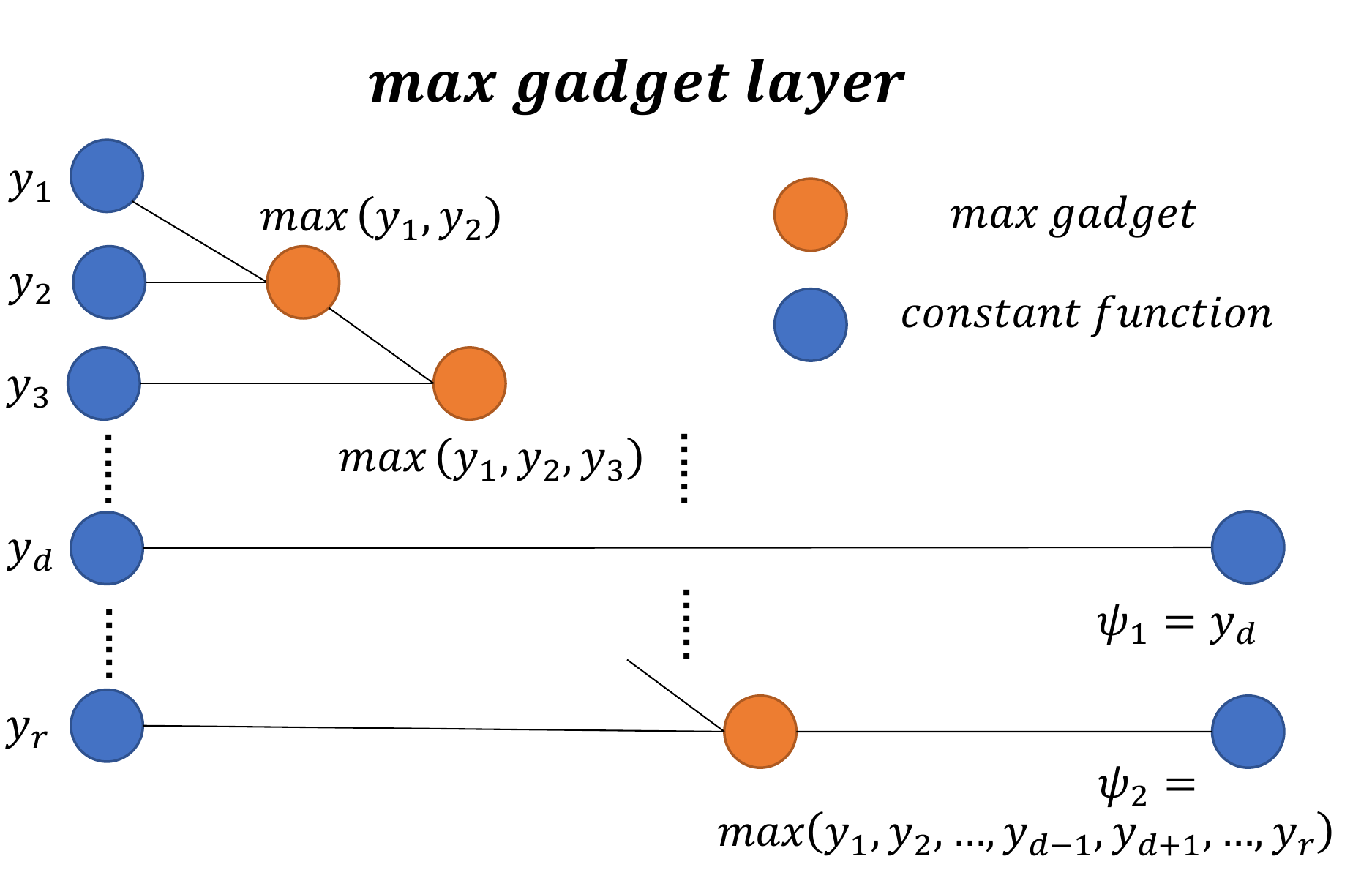}
% \vspace{-2mm}
\caption{The max gadget layer on ReLUs}
\label{fig:max-gadget-layer}
% \vspace{-4mm}
\end{minipage}
\end{figure}

\paragraph{The problem of verifying the local robustness of ReLU-based OMNNs is NP-complete}
\ 

Let $O:\mathbb{R}^{4+ct}\rightarrow\mathbb{R}^{m\times n}$ be the occlusion neural network, whose construction can be seen in Sec. \ref{subsec:encoding}.

Let $M:\mathbb{R}^r\rightarrow\mathbb{R}^2$ be the max gadget layer. For a given constant $d$ and a vector with $r$ elements where $Y$ represents the vector and $y_i$ denotes the $i$-th element in $Y$, it can return $\psi_1=y_d$ and the max value among other elements, namely, $\psi_2=max(y_1,y_2,\ldots,y_{d-1},y_{d+1},\ldots,y_r)$.

Fig. \ref{fig:max-gadget} shows a max gadget based on ReLUs. It can return the larger value between two inputs. The max gadget is only suitable for non-negative values. The max gadget is capable enough because ReLUs will return non-negative values for any inputs..

Fig. \ref{fig:max-gadget-layer} shows the max gadget layer, consisting of max gadgets. It can return the value of $y_d$ and the max value among the others $max\left(y_1,y_2,\ldots,y_{d-1},y_{d+1},\ldots,y_r\right)$.

\begin{definition}[Local robustness of ReLU-based OMNNs]
\label{def:robustness-of-omnn}

Given a ReLU-based FNN $F:\mathbb{R}^{m\times n}\rightarrow\mathbb{R}^r$, an occlusion neural network $O$, and a max gadget layer $M$, the neural network $S=M\circ\ F\circ\ O$ is called a ReLU-based OMNN, i.e., ReLU-based FNN with the occlusion neural network and the max gadget layer.

Let $ \mathrm{\Theta}$(resp. $\mathrm{\Psi}$) be the input(resp. output) vector of $S$ and $\theta_i$(resp. $\psi_i$) be the $i$-th element of $\mathrm{\Theta}$(resp. $\mathrm{\Psi}$). A ReLU-based OMNN $S$ is not locally robust for a $m\times\ n$ image $x$ and correct classification label $d=\mathrm{\Phi}\left(F\left(x\right)\right)$, if there exists $l_i\le\theta_i\le r_i\ for\ all\ i=1,2,\ldots,4+ct$, such that $\psi_1-\psi_2\le0$, where $l_i$ and $r_i$ are given constants; otherwise, the OMNN $S$ is locally robust.

\end{definition}

Definition 4 means that when there exists an input $\mathrm{\Theta}$ within the hyperrectangle of input space, letting $(F\circ\ O\left(\mathrm{\Theta}\right))_d$ not be the largest value among $F\circ\ O\left(\mathrm{\Theta}\right)$, the OMNN is not locally robust and has an unrobust counterexample $\mathrm{\Theta}$. We can also learn that the problem of verifying the local robustness of ReLU-based OMNNs is to determine whether a ReLU-based OMNN $S=M\circ\ F\circ\ O$ is locally robust or not.

The NNReach problem is a reachability problem of determining whether an input specification $\varphi_{in}\left(x_1,x_2,\ldots,x_n\right)$ and an output specification $\varphi_{out}\left(F\left(x_1,x_2,\ldots,x_n\right)\right)$ are both satisfied for a given neural network $F$, where a specification $\varphi\left(x_1,x_2,\ldots,x_n\right)$ is a system of linear constraints on variables $x_1,x_2,\ldots,x_n$. The NNReach problem of ReLU-based DNNs, or judging whether properties satisfy on a ReLU-based DNN is proven to be an NP-complete problem \cite{katz2017reluplex,salzer2021reachability}.

\setcounter{lemma}{0}
\renewcommand{\thelemma}{2.\arabic{lemma}}
\begin{lemma}
\label{lemma:npc-of-omnn}
    The problem of verifying the local robustness of ReLU-based OMNNs is NP-complete.
\end{lemma}
\setcounter{lemma}{1}
\renewcommand{\thelemma}{\arabic{lemma}}

\begin{proof}
According to the Definition \ref{def:robustness-of-omnn} and our encoding approach in Sec. \ref{subsec:encoding}, there is a ReLU-based neural network $S=M\circ\ F\circ\ O$ with input specification $\varphi_{in}\left(\mathrm{\Theta}\right): 1\le\theta_1\le n\land\theta_2=w\land1\le\theta_3\le m\land\theta_4=h\land\bigwedge_{i=5}^{4+ct}{\theta_i=l_i}$ and output specification $\varphi_{out}\left(\mathrm{\Psi}\right): \psi_1-\psi_2\le0$. Now the problem of verifying the local robustness of ReLU-based OMNNs is to determine whether there is $\mathrm{\Theta}\in\mathbb{R}^{4+ct}$ such that $\varphi_{in}\left(\mathrm{\Theta}\right)$ and $\varphi_{out}\left(\mathrm{\Psi}\right)$ are true, where $ \mathrm{\Psi}=S\left(\mathrm{\Theta}\right)$, which conforms to the definition of the NNReach problem. The problem of verifying the local robustness of ReLU-based OMNNs is an instance of the NNReach problem. Since the NNReach problem of ReLU-based DNNs is NP-complete, the problem of verifying the local robustness of ReLU-based OMNNs is also NP-complete.
\end{proof}

\paragraph{The reduction from verifying the local robustness of ReLU-based OMNNs to verifying the local occlusion robustness of ReLU-based FNNs}
\ 

Considering an instance of verifying local robustness in ReLU-based OMNNs with the OMNN $S=M\circ\ F\circ\ O$ and the network input $\mathrm{\Theta}$, where $1\le\theta_1\le n,\ \theta_2=w,\ 1\le\theta_3\le m,\ \theta_4=h$ and $\theta_i=l_i,\ for\ i=5,6,\ldots,4+ct$.

For verifying local occlusion robustness of ReLU-based FNNs, there is an FNN $F$, occlusion function $\gamma_{\zeta,\ w\times\ h}$, the coloring function $\zeta$ and occlusion size $w\times h$, the top-left point coordinate of occlusion rectangle $\left(a,b\right)$, and an input image $x$ with size $m\times\ n$.

We can reduce the former problem to the later one by assigning $a$ to be $\theta_1$, $w$ to be $\theta_2$, $b$ to be $\theta_3$, and $h$ to be $\theta_4$. The FNN $F$ can also be derived from $S=M\circ\ F\circ\ O$ by removing layers $M$ and $O$. The element of input matrix $x$ is in the bias in the last layer of occlusion neural network $O$, as explained in Sec. \ref{subsec:encoding}. After that, since we have already known $a$, $b$, $w$ and $h$, in order to get the coloring function $\zeta$, for multiform case, we assign the occlusion $\vartheta$ to be $\theta_i,\ for\ i=5,6,\ldots,4+ct$, so that $\zeta$ can be computed in the formula in Definition \ref{definition:occlusion-robustness} . And for the uniform case, we let $\zeta\left(x,i,j\right)=\mu$ where $\mu$ is the sum of the diagonal element in the weight and its according element in the bias in the last layer, e.g., $W_4\left(1,1\right)+{\sf b}_4\left(1\right)$. Finally, the occlusion function $\gamma$ has already had all the necessary arguments $x$, $a$, $b$, $w$, $h$ and $\zeta$.  In short, we have all the arguments needed to verify the local occlusion robustness of ReLU-based FNNs, and all the reduction steps above can finish in polynomial time. Therefore, we have shown that the problem of verifying the local occlusion robustness of ReLU-based FNNs is NP-hard. Thurs, we proved the Lemma \ref{lemma:np-hardness} below.

\begin{lemma}
\label{lemma:np-hardness}
    The problem of verifying the occlusion robustness of ReLU-based FNNs is NP-hard.
\end{lemma}

\vspace{-5mm}
\subsection{NP-completeness}

Since we have shown that the problem of verifying the local occlusion robustness of ReLU-based FNNs is both in NP and NP-hard by Lemma \ref{lemma:np-membership} and Lemma \ref{lemma:np-hardness}, it is NP-complete.

\begin{theorem}
\label{theorem-app:np-complete}
The problem of determining whether a ReLU-based FNN $F$ is locally occlusion robust on input image $x$ w.r.t. an occlusion function  $\gamma_{\zeta,w\times h}$ is NP-complete.
\vspace{-2mm}
\end{theorem}

\vspace{-2mm}
\section{Proof for Lemma \ref{definition:correctness of the encoding approach}}\label{sec:lemma3}
\vspace{-1mm}

In this section, we give the full proof for Lemma \ref{definition:correctness of the encoding approach}. The proof for situations of fully occluded and not occluded pixels are discussed in Sec. \ref{subsec:correctness of encoding} and the proof for the real number position occlusion has a very similar proof sketch. 
First recall the formal definition of Lemma \ref{definition:correctness of the encoding approach} in this paper:

\setcounter{lemma}{0}
\begin{lemma}
	Given an occlusion function $\gamma_{\zeta,w\times h}:\mathbb{R}^{m\times n}\times \mathbb{R}\times \mathbb{R}\rightarrow \mathbb{R}^{m\times n}$ and an input image $x$, 
	let $O_{\gamma,x}:\mathbb{R}^{4+ct}\rightarrow \mathbb{R}^{m\times n}$ be the corresponding occlusion neural network. There is $\gamma_{\zeta,w\times h}(x,a,b)=O_{\gamma,x}(a,w,b,h,\zeta)$ for all $1\leq a\leq n$ and $1\leq b\leq m$. 
\end{lemma}

\begin{proof}

Suppose pixel $p$ at position $(i, j)$ is partially occluded by a real number position occlusion at $(a_0, b_0)$, we divide our proof into two parts according to the number of elements in $\mathbb{I}_p$ discussed in Sec. \ref{section:occlusion perturbation}. We show that $O_{\gamma,x}(a,w,b,h,\zeta)_{i,j}=\gamma_{\zeta,w\times h}(x,a,b)_{i,j}$ holds in both cases where the pixel $p$ has one or two elements in $\mathbb{I}_p$. Note that there is no case where there are three elements in $\mathbb{I}_p$.
For uniform occlusions, the coloring function $\zeta$ can be reduced to $\zeta(x, i, j)=\mu,\ \mu \in [0, 1]$, and for multiform occlusions, $\zeta$ is defined as $\zeta(x, i, j)=x_{i,j} + \Delta p,\ \Delta p \in [-\epsilon, \epsilon]$.
% The fourth layer calculate the value of $()$ and once $s$ has the same value, the output has the same value.

It suffices to prove that the $(i\times n + j)^{th}$ output of the third layer of occlusion network $O_{\gamma, x}^{(3)}(a, w, b, h, \zeta)_{i, j}$ is equivalent to the $s_{i,j}$ in $\gamma_{\zeta, w\times h}(x, a, b)_{i, j}$ for all $i\in \mathbb{N}_{1, n}$ and $j\in \mathbb{N}_{1, m}$.
The remaining process is just a forward propagation to propagate the output of the third layer to the fourth layer and it is the same for both cases and for uniform and multiform occlusions, just as described in Section \ref{subsec:correctness of encoding}. A minor difference is that for the multiform occlusion, only when the occlusion is at real number positions, we leverage the equality of $a\times b = exp(log(a) + log(b))$ for multiplication and add it to the propagation of the third layer to guarantee the correctness in this case. Therefore it is straightforward that the fourth layer outputs the correct color once $O_{\gamma, x}^{(3)}(a, w, b, h, \zeta)_{i, j} = s_{i,j}$ is proved, which means $O_{\gamma,x}(a,w,b,h,\zeta)_{i,j} = \gamma_{\zeta,w\times h}(x,a,b)_{i,j}$ holds for all $i\in \mathbb{N}_{1, n}$ and $j\in \mathbb{N}_{1, m}$.

\proofpart{When an arbitrary pixel $p$ has one element $(i', j')\in \mathbb{I}_p$, we have $s_{i,j}=max(0, |1-j-j'| + |1-i-i'| - 1)$. We show that $O_{\gamma, x}^{(3)}(a, w, b, h, \zeta)_{i, j} = s_{i,j}$.}

In case 1, we can observe that for pixel $p$, $(0 \le a_0 - i < 1 \vee 0 \le i - (a_0 + w_0 - 1) < 1)$ and $(0 \le b_0 - j < 1 \vee 0 \le j - (b_0 + h_0 - 1) < 1)$ holds. 
We show the output of $O_{\gamma, x}^{(3)}(a, w, b, h, \zeta)_{i, j}$ by inspecting the $(i\times n + j)^{th}$ output of the occlusion network after propagation to the third layer, starting from inspecting the output of the $i^{th}$ and $(i + m)^{th}$ neurons of the first layer. According to the network structure in Fig. \ref{fig:occNN}, these two neurons calculate the value of $(a - i)$ and $(1 + i - a - w)$, which is in $[0, 1)$. After the propagation to the second layer, the $i^{th}$ neuron in the first part of the second layer outputs the value of $max(0, 1 - i + i')$, which in this case, is equivalent to $|1-i+i'|$. We can derive that the $j^{th}$ output of the second part of second layer is $|1-j-j'|$ through the same process.

In the third layer, the $(i\times n + j)^{th}$ neuron is based on the $i{th}$ and $j^{th}$ neurons of the second layer we discussed above. It adds these two neurons and outputs $max(0, |1-i+i'| + |1-j+j'| - 1)$. Since we only have one element in $\mathbb{I}_p$, it is the same value as the occlusion factor $s_{i,j}$ in $\gamma_{\zeta,w\times h}(x, a, b)_{i, j}$. 

\proofpart{When an arbitrary pixel $p$ has two elements in $\mathbb{I}_p$, we show $O_{\gamma,x}^{(3)}(a,w,b,h,\zeta)_{i,j}=s_{i,j}$ holds where $s_{i,j}=max(0, \textstyle{\sum_{{i'_0, j'}\in \mathbb{I}_{p}}(|1-j+j'|)} + \textstyle{\sum_{i', j'_0\in \mathbb{I}_{p}}(|1-i'+i|)-1})$.}

In this case, we can observe that $(0 \le a_0 - i < 1 \wedge (j\ge b_0 \wedge j < b_0 + h_0 - 1))$ or $(0 \le b_0 - j < 1 \wedge (i\ge a_0 \wedge i < a_0 + w_0 - 1))$ holds for pixel $p$. We take the first equation as an example without loss of generality.
The $j^{th}$ output of the second part in the second layer sums up the effect of the two surrounding occlusion pixels in the same dimension and outputs $\textstyle{\sum_{{i', j'}\in \mathbb{I}_{p}}(|1-j+j'|)}$. It has a value of 1 since pixel $p$ is fully occluded in this dimension. The $i^{th}$ neuron of the second layer outputs $|1-i-i'|$. It is the same output as in case 1.
The $(i\times n + j)^{th}$ neuron in the third layer outputs $max(0, \textstyle{\sum_{{i', j'}\in \mathbb{I}_{p}}(|1-j+j'|)} + |1-i+i'| - 1)$. Notably, the second part in $s$, $\textstyle{\sum_{{i', j_0'}\in \mathbb{I}_{p}}}$, has only one element in the summation. Therefore, $s_{i,j}$ can be reduced to the same formula as the output of $O_{\gamma, x}^{(3)}(a, w, b, h, \zeta)_{i, j}$.\qed 
\vspace{-2mm}
\end{proof}

\section{Supplementary experiments}

Table \ref{Network structure} shows the settings of the neural networks used in our experiments. We verified three neural networks in different sizes on each dataset. The accuracy of each network is comparable to those state-of-the-art networks. 
\begin{table}[h!]
	\centering
	\vspace{-6mm}
	\caption{The FNNs in different scales used in the experiments.}
	\label{Network structure}
	\begin{threeparttable}
			\begin{tabular}{p{1.5cm}|p{1.5cm}|p{5.8cm}|r|r} 
					\toprule
					\centering\textbf{Dataset} & \centering \textbf{Model Size} & \centering \textbf{Structure}  & \begin{tabular}[c]{@{}l@{}}\textbf{\# of ReLUs}\\\end{tabular} & \textbf{Accuracy}  \\ 
					\midrule
					\multirow{3}{*}{MNIST} & Small & $\langle 784, 50, 20, 10\rangle^*$  & 70& 96\%\\
					& Medium& $\langle784, 200, 200, 200, 10\rangle^*$ & 600  & 98\%\\
					& Large & $\langle784, 400, 200, 200, 200, 100, 10\rangle^*$& 1100 & 99\%\\\hline 
					\multirow{3}{*}{GTSRB} & Small & $\langle3072, 50, 20, 7\rangle^*$  & 70& 83\%\\
					& Medium& $\langle3072, 200, 100, 43, 7\rangle^*$  & 343  & 79\%\\
					& Large & $\langle3072, 400, 200, 200, 200, 200, 100, 7\rangle^*$ & 1300 & 81\%\\
					\bottomrule
				\end{tabular}
			\begin{tablenotes}
					\item[*] $\langle k_0,k_2,...k_n\rangle$ means network has $k_0$ neurons in the input layer, $k_i\ (1\leq  i <n)$ neurons in the  $i^{th}$ hidden layer, and $k_n$ neurons in the output layer, respectively. 
				\end{tablenotes}
		\end{threeparttable}
\vspace{-6mm}
\end{table}

\vspace{-2mm}
\subsection{Extra Experiment Results of Experiment I}

Table \ref{table:expI-gtsrb} and Table \ref{table:expI-mnist} show the full experiment data of robustness verification against multiform occlusion on the remaining networks of different scale on MNIST and GTSRB.
We show the data of the small and large FNNs verified with $\epsilon=[0.1, 0.2, 0.3, 0.4, 0.5]$. The data on these two FNNs with too small $\epsilon$ are very similar, therefore we conduct experiments on larger $\epsilon$. For the small and large FNN on MNIST, with $\epsilon=1.0$ and size equals to 2, it has 7 and 5 unrobust cases out of 30, respectively.

\vspace{-1mm}

\begin{table}[h!]
	\centering
	\caption{Occlusion verification results on the remaining two FNNs trained on GTSRB in different occlusion sizes $2\times 2$ and $5\times 5$ and occlusion radius $\epsilon$.}
	\label{table:expI-gtsrb}
	\setlength{\tabcolsep}{1mm}
	\footnotesize 
	\begin{threeparttable}
		\begin{tabular}{|l|c|R{1cm}|rrrr|R{1cm}|rrrr|} 
			\hline 
			&  \multicolumn{1}{c|}{}  & \multicolumn{5}{c|}{Small FNN (70 ReLUs) on GTSRB}  & \multicolumn{5}{c|}{Large FNN (1300 ReLUs) on GTSRB} \\ 
			\hline
			Size & \multicolumn{1}{c|}{$\epsilon$} & \multicolumn{1}{c|}{- / +} & \multicolumn{1}{c|}{$T_{+}$} & \multicolumn{1}{c|}{$T_{-}$} & \multicolumn{1}{c|}{$T_{\text{build}}$} & TO (\%) & \multicolumn{1}{c|}{- / +} & \multicolumn{1}{c|}{$T_{+}$} & \multicolumn{1}{c|}{$T_{-}$} & \multicolumn{1}{c|}{$T_{\text{build}}$} & TO (\%)  \\ 
			\hline
			\multirow{5}{*}{$2\times 2$} 			
			&0.05  & \textbf{5} / 18  & 66.64  & 17.39 & 0.079  & 0.00 & \textbf{9} / 13  & 167.48  & 44.98 & 0.098  & 0.00  \\
			&0.10  & \textbf{6} / 17  & 71.56 & 24.2 & 0.078 & 0.00 & \textbf{10} / 12  & 173.16  & 41.51 & 0.098  & 0.00  \\
			&0.20  & \textbf{9} / 14  & 77.11  & 24.56 & 0.077  & 0.00 & \textbf{11} / 11  & 181.44  & 44.17 & 0.097  & 0.04  \\
			&0.30  & \textbf{11} / 12  & 81.87  & 23.8 & 0.083  & 0.00 & \textbf{12} / 10  & 195.24  & 49.94 & 0.095  & 0.4  \\
			&0.40  & \textbf{14} / 9  & 84.25  & 20.63 & 0.093  & 0.00 & \textbf{12} / 10  & 211.33  & 45.06 & 0.102 & 1.07  \\
			\hline
			\multirow{5}{*}{$5\times 5$} 
			&0.05  & \textbf{9} / 14  & 80.54 & 36.32 & 0.099  & 0.00 & \textbf{11} / 11  & 170.66  & 44.06 & 0.097  & 0.00  \\
			&0.10  & \textbf{13} / 10  & 88.13 & 18.32 & 0.099 & 0.00 & \textbf{13} / 9  & 182.4  & 48.62 & 0.097  & 0.26  \\
			&0.20  & \textbf{19} / 4  & 92.83  & 24.32 & 0.097  & 0.00 & \textbf{16} / 6  & 190.84  & 56.45 & 0.096  & 0.58  \\
			&0.30  & \textbf{20} / 3  & 99.29  & 22.07 & 0.095  & 0.00 & \textbf{18} / 4  & 219.88  & 53.54 & 0.099  & 1.55  \\
			&0.40  & \textbf{23} / 0  & /  & 21.15 & 0.101  & 0.00 & \textbf{21} / 1  & 299.91  & 42.89 & 0.098  & 0.75  \\
			\hline 			
		\end{tabular}
		\begin{tablenotes}
			\footnotesize
			\item[*] - / +: the numbers of non-robust and robust cases; $T_{+}$ (\textit{resp.} $T_{-}$): average verification time in robust (\textit{resp.} non-robust) cases; $T_{\text{build}}$: the building time of occlusion neural networks; TO(\%): the proportion of cases running out time.  
		\end{tablenotes}
	\end{threeparttable}
% 	\vspace{-3mm}
\end{table}

\begin{table}[h!]
	\centering
	\caption{Occlusion verification results on the remaining two FNNs trained on MNIST in different occlusion sizes $2\times 2$ and $5\times 5$ and occlusion radius $\epsilon$.}
	\label{table:expI-mnist}
	\setlength{\tabcolsep}{1mm}
	\footnotesize 
	\begin{threeparttable}
		\begin{tabular}{|l|c|R{0.75cm}|rrrr|R{0.85cm}|rrrr|} 
			\hline 
			&  \multicolumn{1}{c|}{}  & \multicolumn{5}{c|}{Small FNN (70 ReLUs) on MNIST}  & \multicolumn{5}{c|}{Large FNN (1100 ReLUs) on MNIST} \\ 
			\hline
			Size & \multicolumn{1}{c|}{$\epsilon$} & \multicolumn{1}{c|}{- / +} & \multicolumn{1}{c|}{$T_{+}$} & \multicolumn{1}{c|}{$T_{-}$} & \multicolumn{1}{c|}{$T_{\text{build}}$} & TO (\%) & \multicolumn{1}{c|}{- / +} & \multicolumn{1}{c|}{$T_{+}$} & \multicolumn{1}{c|}{$T_{-}$} & \multicolumn{1}{c|}{$T_{\text{build}}$} & TO (\%)  \\ 
			\hline
			\multirow{5}{*}{$2\times 2$} 			
			&0.10  & \textbf{1} / 29  & 68.82  & 26.18 & 0.071 & 0.00 & \textbf{0} / 30  & 190.37  & / & 0.088  & 0.00  \\
			&0.20  & \textbf{1} / 29  & 80.06  & 11.7 & 0.072  & 0.00 & \textbf{0} / 30  & 195.93  & / & 0.087  & 0.00  \\
			&0.30  & \textbf{1} / 29  & 95.77  & 30.12 & 0.069  & 0.00 & \textbf{0} / 30  & 213.37  & / & 0.084  & 0.15  \\
			&0.40  & \textbf{1} / 29  & 97.92  & 15.68 & 0.066  & 0.00 & \textbf{2} / 28  & 277.13  & 89.9 & 0.084  & 0.34  \\
			&0.50  & \textbf{1} / 29  & 108.61  & 12.66 & 0.066  & 0.08 & \textbf{3} / 27  & 394.63  & 115.3 & 0.085  & 1.41  \\
			\hline
			\multirow{5}{*}{$5\times 5$} 
			&0.10  & \textbf{1} / 29  & 78.45 & 24.06 & 0.063 & 0.00 & \textbf{1} / 29  & 188.57  & 147.21 & 0.082  & 0.04  \\
			&0.20  & \textbf{3} / 27  & 83.65  & 15.66 & 0.073  & 0.00 & \textbf{4} / 26  & 239.14  & 102.51 & 0.086  & 2.21  \\
			&0.30  & \textbf{8} / 22  & 106.74  & 19.88 & 0.0677  & 0.00 & \textbf{5} / 25  & 385.37  & 120.61 & 0.085  & 9.15  \\
			&0.40  & \textbf{21} / 9  & 156.98  & 41.16 & 0.066  & 1.56 & \textbf{6} / 24  & 489.23  & 184.57 & 0.084  & 15.26  \\
			&0.50  & \textbf{25} / 5  & 176.02 & 42.98 & 0.068  & 2.83 & \textbf{17} / 13  & 696.92  & 200.12 & 0.118  & 23.14  \\
			\hline 			
		\end{tabular}
	\end{threeparttable}
% 	\vspace{-3mm}
\end{table}

\subsection{Extra Experiment Results of Experiment II}
\label{appendix:full data of experiment 2}
Table \ref{table:Timeout proportion comparison between OccRob and Naive method} shows the timeout proportion of \occrob and the naive approach on the verification against uniform occlusions. We can find that the naive method almost cannot scale to large networks, while \occrob can finish more than 92\% verification queries. 

\begin{table}[h!]
	\centering
% 	\vspace{-8mm}
	\caption{Timeout proportion comparison between OccRob and Naive method}
	\label{table:Timeout proportion comparison between OccRob and Naive method}
	\setlength{\tabcolsep}{2.2mm}
	\begin{tabular}{lccc|ccc} 
		\toprule
		\multirow{2}{*}{\diagbox{TO (\%) of }{ Model}} & \multicolumn{3}{c}{MNIST}      & \multicolumn{3}{c}{GTSRB}  \\
		& Small  & Medium  & \multicolumn{1}{c}{Large} & Small  & Medium  & Large       \\ 
		\midrule
		OccRob& 0.75 & 1.39 & 2.59 & 0.06 & 2.59 & 7.03      \\
		Naive encoding  & 8.89 & 76.26 & 98.24 & 33.43 & 72.53 & 96.57      \\
		\bottomrule
	\end{tabular}
\end{table}

\end{document}